\documentclass[runningheads]{llncs}
\usepackage[T1]{fontenc}
\usepackage{graphicx}
\usepackage{booktabs}
\usepackage[misc]{ifsym}
\newcommand{\corr}{(\Letter)}
\usepackage{caption}
\usepackage{subcaption}
\usepackage{tabularx}
\usepackage[svgnames,table]{xcolor}
\usepackage{amsmath}
\usepackage[colorlinks=true,
  linkcolor=DarkRed,
  citecolor=Blue,
  urlcolor=DarkCyan
]{hyperref}

\newcommand{\revision}[1]{{{\color{black}{#1}}}}

\def\ie{\emph{i.e., }}
\def\eg{\emph{e.g., }}

\spnewtheorem*{proofsketch}{Proof sketch}{\itshape}{\rmfamily}

\begin{document}

\title{The Illusion of Improvement:\\
Reject Inference Strategies in Credit Scoring}

\titlerunning{The Illusion of Improvement: Reject Inference Strategies in Credit Scoring}
\author{Bruno Scarone\inst{1}\corr \and
Ricardo Baeza-Yates\inst{2} %
}
\authorrunning{Scarone and Baeza-Yates}

\institute{%
Northeastern University, Boston, MA, USA \email{scarone.b@northeastern.edu}
\and
KTH Royal Institute of Technology, Stockholm, Sweden \email{rbaeza@acm.org}
}

\tocauthor{Bruno Scarone, Ricardo Baeza-Yates}
\toctitle{The Illusion of Improvement: Reject Inference Strategies in Credit Scoring}

\maketitle              %

\begin{abstract}
Reject inference methods are widely used to mitigate survival bias in credit scoring, yet their effectiveness remains poorly understood. We systematically evaluate several such methods and uncover a structural failure mode: in a natural retraining cycle, models whose accuracy improves while recall collapses create an illusion of improvement that leads practitioners to believe the system is getting better when, in fact, its \revision{rejection quality}---the ability to correctly screen out defaulters---is deteriorating. We then propose a controlled exploration strategy that breaks the feedback loop without statistical assumptions: the lender deliberately approves a fraction of rejected applicants and observes their true outcomes. 
We show that accuracy and rejection quality respond in opposite directions to exploration: accuracy is lower under full exploration than under no exploration, while rejection quality is higher. A lender optimizing accuracy on the pool it has itself assembled finds no exploration rate that improves it and will therefore not explore, which is the worst choice for learning. Even minimal exploration rates (2\%--5\%) prove sufficient in our experiments to diagnose the severity of the feedback loop at near-zero cost. 
Our findings are consistent across two machine learning methods and three real-world datasets, and suggest that standard evaluation protocols are inadequate for assessing models trained under survival bias.

\end{abstract}

\section{Introduction}
\label{sec:intro}

Nowadays many important decisions that affect people's lives are being taken with the support of predictive machine learning (ML) models with little or even no human intervention. Typical use cases include lending, hiring, scholarships, government subsidies, and even criminal justice, among others. In this context, these tools may discriminate against or favor certain individuals, producing significant but poorly understood social impacts. For this reason, under the new EU regulation for the use of AI \cite{AIAct}, most of these tools are classified as high-risk. 

In addition, these predictive tools are based on several assumptions that weaken their scientific validity. First, there is the belief that the available data reasonably represent the problem being solved (dataism \cite{dataism2013}), an assumption that in many cases is debatable and varies by instance. 
This creates what is known as a {\em reality gap} \cite{hildebrandt2021}. Second, in some of the use cases, a new individual's behavior is assumed to be predictable from data of ``similar'' people. This may hold on average, but certainly does not work for all individuals, especially considering that the similarity metric is another proxy for reality. Third, we typically evaluate system success through aggregate metrics (\eg accuracy). This implies that all errors carry the same impact, which is false in most cases.

Given these problems, consider the following hypothetical scenario. Imagine that a bank has built an ML model based on historical data to predict whether a new client will repay a loan. That is, we  assume that data from other people can be used to predict the behavior of a new client. This hypothesis is a simplifying assumption that may hold on average, but certainly does not hold for all individuals. If the prediction is positive, the applicant obtains a loan; otherwise, the application is rejected. After a certain period, as the bank accumulates more data, the model can be retrained, with the expectation of improving it. However, the new data suffer from survival or survivorship bias with respect to the model's errors. Specifically, we observe the outcomes of all false positives (favored individuals who represent the survival bias) but learn nothing about the false negatives (rejected or harmed individuals). This process, in addition to neglecting the bias issue described above, allows us to estimate the size of the reality gap.

In this work, using real data, we simulate this bank lending scenario and measure the information lost over time due to survival bias (false negatives), its impact on model performance, and the effectiveness of survival bias mitigation techniques, known as reject inference methods. To this end, we compare the model over time against an Oracle that has access to all available data and hence makes fewer mistakes. We also propose a controlled exploration technique that allows the bank to estimate the %
impact of survival bias at almost no cost.

The contributions of our work are summarized as follows:

\noindent
\textbf{(1)} We uncover a structural failure mode in iterative lending: models whose accuracy improves while measured recall collapses create an illusion of improvement. We show this is not incidental: standard metrics systematically reward strategies that amplify survival bias, as demonstrated by the Biased baseline outperforming the Oracle on accuracy despite having access to strictly less information.

\noindent
\textbf{(2)} We propose controlled exploration as an assumption-free alternative to reject inference, where the lender deliberately approves a fraction of rejects and observes true outcomes.

\noindent
\textbf{(3)} We show that even minimal exploration (2\%--5\%) diagnoses feedback loop severity at near-zero cost, and that accuracy and rejection quality respond in opposite directions to exploration: a lender optimizing accuracy on the pool it has itself assembled finds no exploration rate that improves it and will therefore not explore, which is the worst choice for learning.

The rest of the paper is organized as follows. We cover related work in Section~\ref{sec:related_work} and explain our methodology and datasets used in Section~\ref{sec:methods_data}. Then we present the reject inference results in Section~\ref{sec:reject-inference} and the exploration strategy in Section~\ref{sec:exploration}, ending with the conclusions in Section~\ref{sec:conclusions}.
\section{Related Work}
\label{sec:related_work}

\noindent
\textbf{Loan Default Prediction.}
Credit scoring and default risk estimation are central to financial lending \cite{he2018novel}, with ML-based approaches gaining popularity \cite{suhadolnik2023machine}. The nature and quality of available data have been questioned \cite{ippoliti2017dark}, and datasets are oftentimes imbalanced, with evidence that models perform better on balanced versions \cite{markov2022credit,alam2020investigation}. Tree-based models are a popular choice for their interpretability \cite{markov2022credit,alam2020investigation,suhadolnik2023machine}, though no consensus exists on the most effective method \cite{markov2022credit}. We evaluate a range of models on three datasets and select the two best-performing (both tree-based). Two datasets are down-sampled to a common imbalance ratio.

\noindent
\textbf{Survival Bias in Loan Default Prediction.}
An extensive list of works has highlighted the diverse types of bias that can affect the ML pipeline \cite{mehrabi2021survey,hellstrom2020bias}. We consider survival bias, a form of selection bias in which one focuses on a subset of a population that passes (``survives'') a given selection criterion. 
Survival bias is well-documented in mutual fund data \cite{wermers1997momentum,elton2015,rohleder2011survivorship} and it has been assessed in self-reported financial datasets \cite{weinblat2018forecasting}, but in those cases it arises from data collection rather than model-driven selection. More broadly, the feedback loop we study is an empirical instance of performative prediction~\cite{perdomo2020performative}, where model deployment shifts the data distribution; analogous runaway (self-reinforcing) effects have been documented in predictive policing~\cite{ensign2018runaway}.
When a credit scoring model is used to accept or reject applicants, outcomes are observed only for accepted applicants. Rejected applicants never receive credit, so whether they would have defaulted or not remains a counterfactual that cannot be verified. 
\revision{\emph{Reject inference} (RI) refers to the family of techniques that attempt to recover this missing information, thereby reducing the survival bias in the training data used for subsequent model retraining~\cite{hand1993can,banasik2003sample,ehrhardt2021reject}. 
The RI literature encompasses methods that impute labels for rejected applicants via probabilistic assignment, direct model extrapolation, nearest-neighbor transfer, or bin-averaged imputation~\cite{banasik2007reject,banasik2003sample,ehrhardt2021reject}, as well as self-learning approaches that iteratively refine imputed labels~\cite{kozodoi2019shallow}, deep generative models~\cite{mancisidor2020deep}, and Monte Carlo simulation frameworks~\cite{anderson2023monte}; Ehrhardt et al.~\cite{ehrhardt2021reject} provide a comprehensive survey.}

\section{Methodology}\label{sec:methods_data}

We consider a bank that uses a binary classifier to predict loan defaults and decide which applicants to approve. The bank operates in a \emph{dynamic} setting: at each lending cycle $t \geq 1$, a new batch of applicants arrives and is screened by $M_{t-1}$, the model trained at the end of the previous cycle; the bank
observes outcomes only for approved loans, and retrains on the accumulated data to obtain $M_t$.
We adopt the usual label convention $y = 1$ for default and $y = 0$ for non-default. Our goal is to evaluate how this selective observation process affects model quality over successive retraining cycles.\footnote{The source code is available at\\ \url{https://github.com/bscarone/survival-bias-credit-scoring}.}

Hence, we compare five RI strategies, evaluated against two baselines: \emph{Biased}, which trains only on accepted applicants, and \emph{Oracle}, which has access to true labels for all applicants (both defined formally in Section~\ref{subsec:baselines}). RI refers to techniques~\cite{hand1993can,banasik2003sample,ehrhardt2021reject} that mitigate survival bias by imputing the missing labels of rejected applicants and adding them to the training set alongside accepted applicants, who retain their true labels.

\noindent
\textbf{Reject Inference Strategies.}
\label{sec:ri-strategies}
\emph{Parceling}~\cite{banasik2003sample,siddiqi2012credit} %
bins applicants by predicted risk score into $K{=}10$ equal-frequency parcels and assigns default labels to rejects within each parcel at the observed default rate among accepts in that parcel, falling back to the default rate among all accepts in the batch for parcels containing fewer than 30 accepts. 
\emph{Fuzzy Augmentation}~\cite{banasik2007reject,ehrhardt2021reject} assigns each reject a label sampled from $\tilde{y}_i \sim \text{Bernoulli}(\hat{p}_i)$, where $\hat{p}_i$ is the model's predicted default probability; we use this probabilistic variant, equivalent in expectation to the original weighted formulation, for compatibility with tree-based classifiers. \emph{Simple Extrapolation}~\cite{banasik2003sample,ehrhardt2021reject} applies the current model to rejected applicants and treats its hard predictions as ground-truth labels. For deterministic models under fixed thresholds, the imputed label of a reject coincides with the model's own prediction on it by construction.
\emph{Twins}~\cite{banasik2003sample,ehrhardt2021reject} transfers labels from accepted to rejected applicants via $k$-nearest-neighbor matching ($k=5$) in standardized feature space. 
\emph{Shallow Self-Learning}~\cite{kozodoi2019shallow} first scores each reject by its similarity to the accepted applicants, using an Isolation Forest fitted on the accepts, and discards the least and most similar deciles: the former lie too far from the accepted region for the weak learner to label reliably, while the latter are judged to add little that the accepts do not already supply. 
It then labels the retained rejects in rounds, using an $L_1$-regularized logistic regression as the weak learner, with imbalance-corrected class weights. The first round labels the top and bottom deciles of the predicted probabilities; later rounds reuse those same two probability values instead of recomputing deciles, so a reject not labeled in the first round is labeled only if retraining moves its predicted probability outside the unchanged window.\footnote{We set $\alpha_{SL}=0.1$ and $\beta_{SL}=0.1$, each applied at both tails---of the predicted probabilities and of the novelty scores respectively---with at most five labeling rounds and a weight of 2 on defaulting examples in the weak learner.}

\noindent
\textbf{Data.} We use three public credit scoring datasets~\cite{he2018novel,markov2022credit}: Default~\cite{default_of_credit_card_clients_350} (30,000 persons, 23 features, 77.9\% non-default),
PPDai~\cite{ppdai} (31,230 persons, 29 features, 77.0\%), and LendingClub~\cite{lendingclube2007_2020Q3} (279,520 persons, 97 features, 77.6\%).
The binary label indicates whether the borrower defaulted.
For LendingClub, we filtered to loans originated in 2017, removed post-hoc data leakage columns (\eg payment outcomes, recovery fields), and removed high-null columns. 
Both PPDai and LendingClub were down-sampled to an imbalance ratio of 77\%; the person counts and imbalance ratios above are those remaining after dropping rows with missing values, which is why LendingClub's differs slightly.
Full preprocessing details are provided in Appendix~\ref{app:preprocessing}.
\noindent
\textbf{Models.} We evaluated five binary classifiers across all three datasets (Appendix~\ref{app:model-selection}), selecting two tree-based models that consistently ranked among the top performers: Gradient Boosted Decision Trees (GBDT) and Random Forest (RF). These models are inherently non-linear and are implemented in scikit-learn~\cite{sklearn_gral} with the default decision threshold $\tau = 0.5$. \revision{Both are tree-based ensemble methods; we discuss the potential influence of model choice on our findings in Section~\ref{sec:conclusions}.}

\noindent
\textbf{Temporal Lending Simulation.}
We simulate iterative model-based lending as follows. An initial training set $\mathcal{D}_0$ of size $n_0$ is drawn uniformly from the complete dataset $\mathcal{D}$ of size $n$, modeling a bank's initial data collection step. The value of $n_0$ is determined per dataset--model pair via learning curve analysis: we train models at increasing sample sizes and select the point at which accuracy stabilizes (Appendix~\ref{app:nzero}). This
yields $n_0 = 5{,}000$ for Default and PPDai, and
$n_0 = 60{,}000$ for LendingClub.
At each subsequent lending cycle $t \in \{1, \dots, T\}$, we take a uniform random sample of size $n' = \lfloor (n - n_0) / T \rfloor$ from the remaining data, treating it as the pool of new applicants. The current model $M_{t-1}$, trained on $D_{t-1}$, produces predictions: applicants with $\hat{y} = 0$ are accepted and those with $\hat{y} = 1$ are rejected. Each strategy defines a different rule for incorporating or excluding rejected applicants from the training set $D_t$ on which $M_t$ is trained.

\noindent
\textbf{Evaluation metrics.}
All metrics are computed at each cycle on a held-out split of the accumulated training set $\mathcal{D}_t$, on which $M_t$ issues predictions of both classes. Concretely, $M_0$ is trained on $\mathcal{D}_0$ and evaluated on a held-out split of it; at each cycle $t \geq 1$, $M_{t-1}$ screens the arriving batch, the strategy determines which applicants enter $\mathcal{D}_t$ and with which labels, and $M_t$ is trained on 80\% of $\mathcal{D}_t$ and evaluated on the remaining 20\%.
Because the strategies differ in how they construct $\mathcal{D}_t$, they differ in what they are evaluated on, in two independent respects: which applicants the split contains, and which labels those applicants carry. 
A strategy that discards rejected applicants is evaluated on a subpopulation shaped by the model's own decisions, and one that imputes labels for rejected applicants is evaluated partly against its own imputations rather than against observed outcomes.
We adopt this convention deliberately: a lender performing reject inference treats imputed labels as data and has no ground truth for rejected applicants, so evaluating against $\mathcal{D}_t$ reflects the practitioner's epistemic position.
It does mean that the five reject inference strategies defined above are not evaluated against a common reference, and that comparisons across them mix differences in predictive quality with differences in what each is scored against. The two baselines of Section~\ref{subsec:baselines} are exempt: Biased and Oracle impute no labels, so every label in their training sets is an observed outcome and the comparison between them isolates the effect of which applicants are retained.
Beyond standard classification metrics (accuracy, precision, recall), we track two quantities designed to capture the dynamics specific to iterative model-based lending. 
The \emph{training set default rate} (TDR) measures the proportion of tuples labeled as defaults in the training set at each iteration: $\text{TDR}_t = |\{i \in \mathcal{D}_t : y_i = 1\}|/{|\mathcal{D}_t|}$, where $y_i$ is the label recorded in $\mathcal{D}_t$.
For Biased and Oracle (Section~\ref{subsec:baselines}) every recorded label is an observed outcome, so TDR is the default rate of the training set in the usual sense; for the reject inference strategies it is the rate at which the label $1$ appears, combining observed outcomes for accepted applicants with imputed labels for rejected ones.
Read either way, TDR is a property of the data the model is trained on, and it can be compared against the population default rate $\pi = P_{\mathcal{D}}(y = 1)$ without access to the true outcomes of rejected applicants.
This is what makes it useful here: deviations from $\pi$ signal that the label distribution the model learns from has diverged from the population, whether because informative observations were discarded or because imputation supplied the wrong labels.

We also assess rejection quality: whether the model's rejections are correctly targeting defaulters rather than creditworthy applicants, as recorded in $\mathcal{D}_t$.
Inspired by Kozodoi et al.~\cite{kozodoi2019shallow}, we define the Kickout metric to measure this quantity.\footnote{Our definition differs from that of Kozodoi et al.~\cite{kozodoi2019shallow}, who define Kickout as a two-model comparison. Our metric evaluates a single model's rejection precision; we retain the name for its intuitive appeal.} 
Let $n_{BR}$ denote the number of rejected applicants labeled as defaults (\emph{bad rejects}) and $n_{GR}$ the number of rejected applicants labeled as non-defaults (\emph{good rejects}): 
\begin{equation}
    \text{KO} = \frac{n_{\text{BR}} - n_{\text{GR}}}{n_{\text{BR}} + n_{\text{GR}}} =
    \frac{TP-FP}{TP+FP} = \frac{P(\hat{y}{=}1 \land y{=}1) - P(\hat{y}{=}1 \land y{=}0)}{P(\hat{y}{=}1)}.
    \label{eq:kickout}
\end{equation}
This metric ranges over $[-1,1]$: $\mathrm{KO}=1$ indicates that every rejected applicant carries the label $y=1$, $\mathrm{KO}=-1$ that every rejected applicant carries $y=0$, and intermediate values interpolate between the two.
In the temporal setting, an increasing KO indicates that rejections are becoming more concentrated on applicants labeled as defaulters, and a decreasing KO that they are becoming less so.
Since $n_{BR}$ and $n_{GR}$ are counted against the labels recorded in $\mathcal{D}_t$, KO measures rejection quality relative to those labels; for the reject inference strategies it is therefore partly a measure of agreement between the model and the strategy's own imputations rather than of alignment with observed outcomes.

Two remarks on the form of Equation~\ref{eq:kickout}. First, the middle expression shows that KO is a rescaling of the rejection precision on the evaluation split: writing
\begin{equation}
\mathrm{KO} = \frac{\mathrm{TP}-\mathrm{FP}}{\mathrm{TP}+\mathrm{FP}}
= \frac{2\,\mathrm{TP}-(\mathrm{TP}+\mathrm{FP})}{\mathrm{TP}+\mathrm{FP}}
= \frac{2\,\mathrm{TP}}{\mathrm{TP}+\mathrm{FP}}-\frac{\mathrm{TP}+\mathrm{FP}}{\mathrm{TP}+\mathrm{FP}}
= 2\cdot\mathrm{prec} - 1 .
\label{eq:ko-precision}
\end{equation}
This is by design: KO restates rejection precision on a scale whose endpoints have a direct interpretation in terms of who is denied credit, which is the quantity of interest in this setting. 
Since Equation~\ref{eq:ko-precision} is increasing in $\mathrm{prec}$, KO and precision rank strategies identically; KO is thus an interpretive lens on precision rather than a second quantity to be read alongside it. 

One reference point is useful for reading absolute KO values.
Recall that $\mathrm{prec} = P(y=1 \mid \hat{y}=1)$, the precision of
$M_t$ on the evaluation split (Appendix~\ref{app:theory}). Suppose rejections were made independently of the true labels, at any rate: then conditioning on $\hat{y}=1$ carries no information about $y$, so $\mathrm{prec} = P(y=1) = \pi$ and, by Equation~\ref{eq:ko-precision}, $\mathrm{KO} = 2\pi - 1$, which equals $-0.56$ on Default, $-0.54$ on PPDai and $-0.55$ on LendingClub. 
Two consequences are worth drawing out. First, $KO > 2\pi - 1$ is equivalent to $\mathrm{prec} > \pi$: rejected applicants carry the default label at a higher rate than the population, which is assumption A5 of Section~\ref{subsec:theory}. The value at which screening stops being informative in this sense is therefore dataset-specific, and on our three datasets it lies between $-0.56$ and $-0.54$, not at $0$. 
Second, $KO > 0$ is equivalent to $\mathrm{prec} > 1/2$: a majority of rejected applicants carry the default label, against a population rate of $\pi$ between $0.22$ and $0.23$. The two thresholds are far apart, and $KO > 0$ is the stronger of them.
Unless stated as relative, differences in these metrics are absolute throughout the paper. 
\section{Reject Inference Under Iterative Survival Bias}
\label{sec:reject-inference}

We evaluate the five reject inference strategies presented in Section~\ref{sec:methods_data} against two baselines (Biased and Oracle) that we introduce below. All seven strategies begin from the same initial training set; they differ in whether and how they incorporate rejected applicants at each lending cycle.
We report results over $T{=}10$ iterations, averaged across five runs. We present Default/GBDT as the primary configuration in the main text; full results for all other configurations, including the effect of dataset size, are reported in Appendix~\ref{app:ri-additional-results} and are consistent with the patterns described here.

\subsection{Baselines: Biased and Oracle}
\label{subsec:baselines}

Before evaluating reject inference strategies, we establish two baselines that represent the extremes of information availability.

The \textbf{Biased} baseline trains exclusively on accepted applicants, discarding all rejected samples. At each cycle $t>0$, the training set consists of the previous data augmented only with the newly accepted applicants and their observed outcomes:
\revision{$\mathcal{D}_t = \mathcal{D}_{t-1} \cup \{(\mathbf{x}_i, y_i) : \mathbf{x}_i \in \mathcal{S}_t,\ \hat{y}_i = 0\}$}, where $\mathcal{S}_{t}$ denotes the full sample of new applicants at cycle $t$.
This is the default approach when no RI is applied and represents the full extent of survival bias accumulation. The \textbf{Oracle} baseline has access to the true outcome $y_i$ for all applicants, both accepted and rejected:
\revision{$\mathcal{D}_t = \mathcal{D}_{t-1} \cup \{(\mathbf{x}_i, y_i) : \mathbf{x}_i \in \mathcal{S}_t\}$}.
This is infeasible in practice, one cannot observe the default behavior of applicants who were never granted credit, but it provides a theoretical upper bound against which all other strategies can be measured.
Their contrasting behavior under iterative retraining (Figure~\ref{fig:biased-oracle}) establishes the degradation pattern that reject inference strategies aim to correct, and provides a first indication that standard metrics do not reliably distinguish good models from biased ones.

\noindent\textbf{Biased baseline.}
The training set default rate decreases monotonically under the Biased strategy, falling from the population rate of 22.1\% to approximately 17\% on Default/GBDT by $t{=}10$ (Figure~\ref{fig:default-rate-kickout}a), a deflation of roughly 5\% over the ten cycles. This deflation reflects the fundamental mechanism of survival bias: as the model rejects applicants it considers risky, it removes precisely the observations that are most informative about default behavior from its own training data.
The consequences for predictive performance are asymmetric (Figure~\ref{fig:biased-oracle}). Accuracy remains stable and even increases slightly, rising from 0.818 to 0.833 on Default/GBDT. However, this apparent stability masks a collapse in recall, from approximately 0.37 to 0.11, as the model progressively loses its ability to detect the defaulters it has excluded from training. The model becomes increasingly confident in predicting non-default for the homogeneous population it observes, while the representation of the default class in its training data shrinks with each iteration.  Kickout falls from 0.27 to 0.19 over the ten cycles (Figure~\ref{fig:default-rate-kickout}b), though not monotonically, and with substantial variation across runs ($\pm 0.11$ at $t=0$, $\pm 0.07$ at $t=10$); the decline holds in four of the five runs. 
Because $\mathcal{D}_t$ under Biased consists of the initial sample $\mathcal{D}_0$ together with the applicants accepted since, this largely measures the precision of the model's rejections among applicants it had already accepted: of the 26{,}934 tuples in $\mathcal{D}_{10}$ on Default/GBDT, 21{,}934 are accepted applicants and 5{,}000 are the unscreened initial sample, while the 3{,}066 applicants rejected over the ten cycles are absent from the set on which its rejection quality is scored. This pattern is consistent across all three datasets and both modeling techniques, though the magnitude varies with dataset scale (Appendix~\ref{app:ri-dataset-size}).

\begin{figure}[t]
    \centering
    \includegraphics[width=\textwidth]{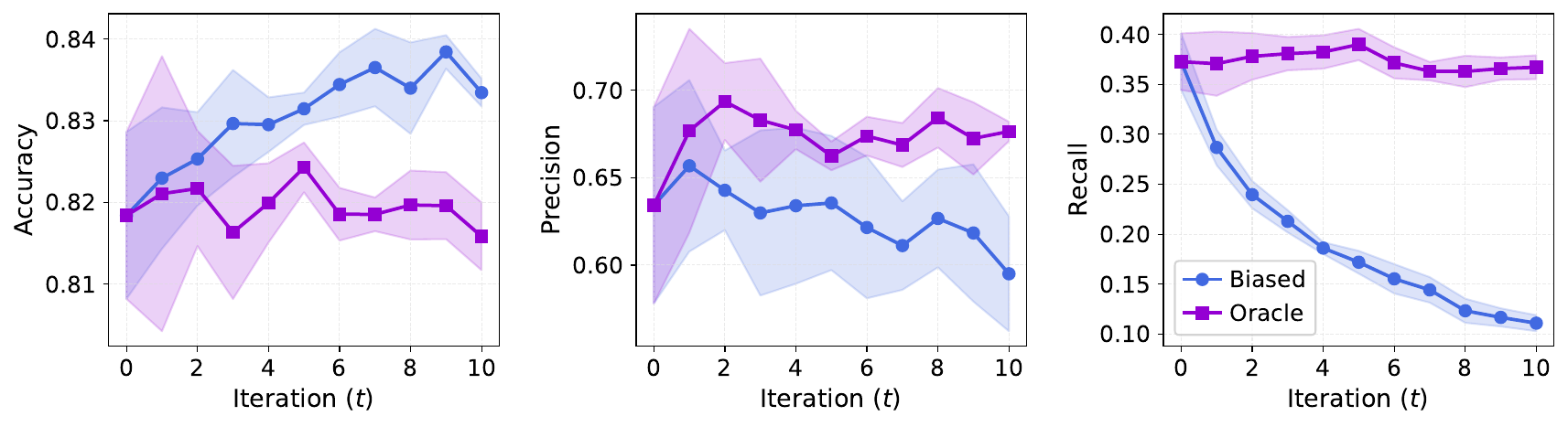}
    \caption{Accuracy, precision, and recall over 10 lending iterations for the Biased and Oracle strategies on the Default dataset (GBDT, $n_0{=}5{,}000$, $\tau{=}0.5$). Shaded regions indicate $\pm 1$ standard deviation across 5 runs. Accuracy remains stable under the Biased strategy while recall collapses from 0.37 to 0.11: the model progressively loses its ability to detect the defaulters it has excluded from training.}
    \label{fig:biased-oracle}
\end{figure}

\noindent\textbf{Oracle baseline.}
With access to the true default labels for all applicants, Oracle maintains a training set default rate of approximately 22.1\%, essentially identical to the population rate. Recall remains near its initial level across all iterations, confirming that access to rejected applicants' true outcomes prevents the information loss that drives the Biased strategy's degradation.

Yet Oracle achieves \emph{lower accuracy} than the Biased baseline: 0.816 versus 0.833 on Default/GBDT at $t{=}10$ (Figure~\ref{fig:biased-oracle}). This is the first sign that standard evaluation metrics may reward survival bias rather than penalize it. By including genuine label diversity, \ie borderline cases and applicants who default despite having features similar to non-defaulters, Oracle makes learning harder and produces a less decisive classifier. The Biased model, by contrast, trains on an increasingly homogeneous population and achieves high accuracy precisely because it has narrowed the problem it needs to solve. This inversion, where more information yields lower measured performance, recurs in the comparison of reject inference strategies in Section~\ref{subsec:ri-comparison}.

\subsection{\revision{Theoretical Analysis of the Accuracy--Recall Paradox}}
\label{subsec:theory}

We formalize the structural failure mode identified empirically in Section~\ref{subsec:baselines} under the following assumptions, where $M_{t-1}$ screens the batch arriving at cycle $t$ and $M_t$ is trained on the resulting $\mathcal{D}_t$:
\textbf{(A1)} At cycle $t$ the fixed-threshold binary classifier $M_{t-1}$ produces $\hat{y} = M_{t-1}(\mathbf{x}) \in \{0,1\}$, where $y=1$ denotes default.\footnote{
We leave the cycle index on $\hat{y}$ implicit. 
In A1--A5 and in Lemma~\ref{lem:cleaning} and Proposition~\ref{prop:onecycle}, $\hat{y}$ is the prediction that screens the batch arriving at cycle $t$, so the screening quantities $\alpha_t$, $\mathrm{prec}_t$, $\beta_t$ depend on $t$ through $M_{t-1}$; in Proposition~\ref{prop:decomp} and its proof, $\hat{y} = M_t(\mathbf{x})$ on the accumulated pool.
}
\textbf{(A2)} At each cycle the new applicants are drawn uniformly, without replacement, and independently of $M_{t-1}$'s past decisions from a fixed dataset $\mathcal{D}$ with base default rate $\pi = P_{\mathcal{D}}(y=1) \in (0, 1/2]$, so survival bias enters only through the training set, not the arrivals. 
\textbf{(A3)} True labels are observed only for accepted applicants ($\hat{y}=0$).
\textbf{(A4)} Screening is non-degenerate: at every cycle $M_{t-1}$ predicts $\hat{y}=1$ (defaulter) for a nonempty proper subset of the arrivals, i.e.\ the rejection rate satisfies $\alpha_t := P_{\mathcal{D}}(\hat{y}=1) \in (0,1)$.
\textbf{(A5)} Screening is informative: on the $t$ arrival cycles, the rejection precision exceeds the base rate, $\mathrm{prec}_t := P_{\mathcal{D}}(y=1 \mid \hat{y}=1) > \pi$; equivalently, a rejected applicant is more likely to default than a randomly drawn one. 
\textbf{(A6)} The model is retrained on the accumulated accepted data without class rebalancing --- the \emph{Biased} baseline of Section~\ref{subsec:baselines} --- and the initial set $\mathcal{D}_0$ is a uniform sample of $\mathcal{D}$, so $\pi_0 = \pi$.
We write 
$\pi_t := P_{\mathcal{D}_t}(y=1)$ for the default rate of the
accumulated pool $\mathcal{D}_t$, 
which is the training set default rate $\text{TDR}_t$ of Section~\ref{sec:methods_data}.\footnote{Quantities differ in two respects that move together: the model they refer to ($M_{t-1}$ or $M_t$) and the population over which they are computed ($D$ or $\mathcal{D}_t$). The screening rates $\alpha_t$, $\mathrm{prec}_t$ and $\beta_t$ are properties of $M_{t-1}$ over a draw from $D$ ($P_D$); the evaluation metrics of Proposition~\ref{prop:decomp} are properties of $M_t$ over the accumulated pool $\mathcal{D}_t$ ($P_{\mathcal{D}_t}$), as is $\pi_t$, which refers to no model. A quantity measured on one population does not in general equal its counterpart on the other. In particular, $\mathrm{prec}_t$ in A5 is the rejection precision of $M_{t-1}$ on a draw from $D$, whereas the $\mathrm{prec}$ of Appendix~\ref{app:theory} is that of $M_t$ on the accumulated pool: the same conditional probability, taken under a different model and a different population.}
Here $\pi$ is exact (the default rate of $\mathcal{D}$). All proofs are deferred to Appendix~\ref{app:theory}.

A5 is verified directly by the exploration experiments of
Section~\ref{sec:exploration}. On Default/GBDT, the default rate among
randomly explored rejects, aggregated over iterations and runs, is
67.1\% at $r = 0.02$ (224 of 334 explored rejects), where the screening model departs only slightly from the Biased process of A6, and 67.1\% at $r = 1.0$ (10{,}159 of 15{,}147), where every rejected applicant is observed and $\mathrm{prec}_t$ is measured rather
than estimated.
Table~\ref{tab:exploration-default-rates} reports the same quantity by
configuration at $r = 0.10$ and $t = 10$, averaged over runs, where it
ranges from 54.6\% to 66.5\%; computed separately for each configuration and iteration, aggregating over values of $r$ and runs, it never falls below 52\%.
We read these rates under random selection, since Most Risky and Least Risky sample a deliberately skewed slice of the rejects and so measure the default rate in that slice rather than in the pool. Hence $\mathrm{prec}_t$ exceeds $\pi$ throughout by a factor of roughly two to three, against population rates of 22\%--23\%.

\begin{lemma}
\label{lem:cleaning}
Under A1--A6, the default rate of the accepted applicants at cycle $t$, $\beta_t := P_{\mathcal{D}}(y=1 \mid \hat{y}=0)$, satisfies $\beta_t < \pi$: the accepted applicants default less often than the population.
\end{lemma}

\begin{proposition}[Single-cycle default-rate deflation]
\label{prop:onecycle}
Under A1--A6, after the first retraining cycle the accumulated pool has a strictly lower default rate than the population: $\pi_1 < \pi_0 = \pi$.
\end{proposition}

\noindent
Since each accepted batch defaults less than the population (Lemma~\ref{lem:cleaning}), mixing the first batch into $\mathcal{D}_0$ drags the pool's default rate below $\pi$ (Proposition~\ref{prop:onecycle}).

\begin{proposition}[Accuracy--recall decomposition on the accumulated pool]
\label{prop:decomp}
On the accumulated pool at cycle $t$ (default rate $\pi_t$), evaluated as in Section~\ref{sec:methods_data}, the accuracy of $M_t$ decomposes as
$\mathrm{acc}_t \;=\; (1-\pi_t)\,\mathrm{spec}_t \;+\; \pi_t\,\mathrm{rec}_t$,
where $\mathrm{spec}_t = P_{\mathcal{D}_t}(\hat{y}=0 \mid y=0)$ and $\mathrm{rec}_t = P_{\mathcal{D}_t}(\hat{y}=1 \mid y=1)$ are the specificity and recall of the model $M_t$ on that pool, with $\hat{y} = M_t(x)$. In particular, the contribution of recall to accuracy is exactly $\pi_t\,\mathrm{rec}_t$, which vanishes as $\pi_t \to 0$. Hence as $\pi_t$ falls, accuracy is increasingly governed by specificity and increasingly insensitive to recall. Thus a decline in recall need not lower accuracy, and can coexist with an increase in it when specificity rises.
\end{proposition}

On Default/GBDT the decomposition accounts for the observed trajectory as follows. The recall term $\pi_t\,\mathrm{rec}_t$ falls from 0.082 at $t=0$ to 0.019 at $t=10$, as recall drops from 0.373 to 0.111 and its weight $\pi_t$ from 0.220 to 0.171. This term contributes 8.2\% of accuracy at $t = 0$ and 1.9\% at $t = 10$, a loss of 6.3 points.
Accuracy nonetheless rises by 1.5 points over the same interval, so the specificity term must gain 7.8 points. Inverting the decomposition, $\mathrm{spec}_t = (\mathrm{acc}_t - \pi_t\,\mathrm{rec}_t)/(1-\pi_t)$, gives a specificity of 0.944 at $t=0$ and 0.982 at $t=10$: the model becomes markedly better at the majority class on a pool it has made more homogeneous. Proposition~\ref{prop:decomp} accounts for the insensitivity of accuracy to the recall collapse; the increase requires this specificity gain, which the decomposition does not predict.

Proposition~\ref{prop:onecycle} establishes the one-cycle deflation of the TDR ($\pi_t$); empirically it compounds across cycles (Figure~\ref{fig:default-rate-kickout}a), driving recall down. Since recall contributes to accuracy only in proportion to $\pi_t$ (Proposition~\ref{prop:decomp}), accuracy can hold even if recall collapses (Figure~\ref{fig:biased-oracle}).

\subsection{Reject Inference Strategies Comparison}
\label{subsec:ri-comparison}

We now examine whether the five reject inference strategies (Section~\ref{sec:ri-strategies})---Parceling, Fuzzy Augmentation, Simple Extrapolation, Shallow Self-Learning, and Twins---can correct the survival bias documented above. Throughout this section we report accuracy, precision and recall as a practitioner would encounter them: computed on each strategy's own accumulated pool, against the labels that pool records. The tiers below are therefore an account of what these metrics say about the strategies, not a ranking by predictive quality; the point of the section is that what they say is unreliable. The results reveal a striking paradox: the strategy that performs best on standard predictive metrics is also the one that most severely distorts the training data.

\begin{figure}[t]
    \centering
    \includegraphics[width=\textwidth]{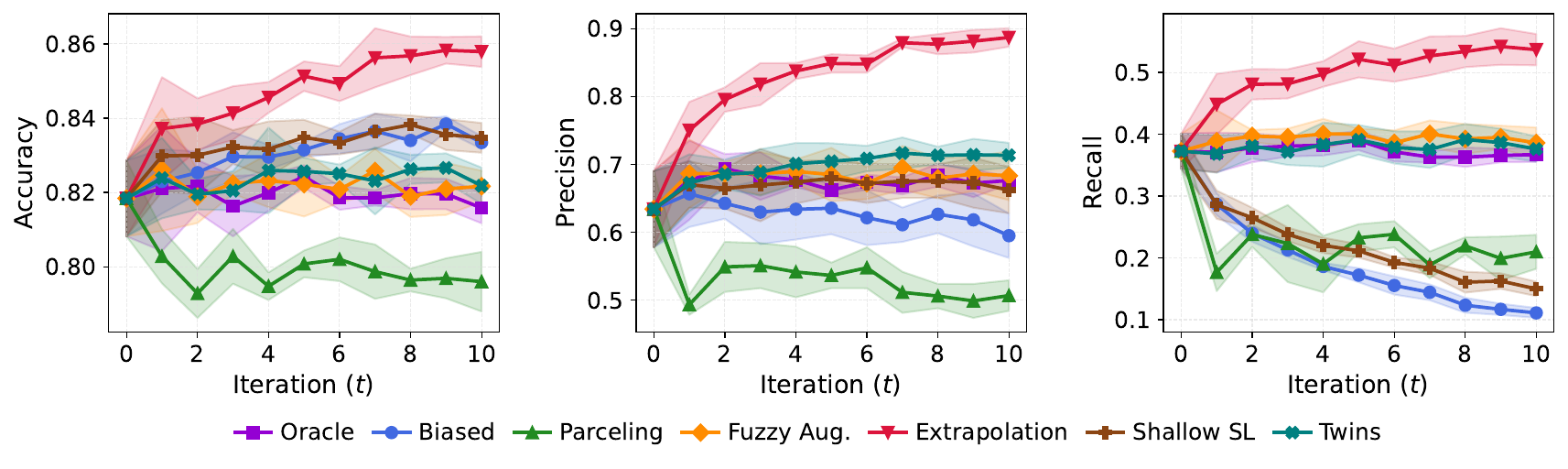}
    \caption{Accuracy, precision, and recall over 10 lending iterations for all seven strategies on the Default dataset (GBDT, $n_0{=}5{,}000$, $\tau{=}0.5$). Shaded regions indicate $\pm 1$ standard deviation across 5 runs. Extrapolation dominates all three standard metrics, yet this apparent superiority is an artifact of evaluating partly against its own imputed labels (Section~\ref{subsec:ri-comparison}).}
    \label{fig:ri-metrics}
\end{figure}

\noindent\textbf{Simple extrapolation dominates standard metrics.}
Extrapolation achieves the highest accuracy, precision, and recall across all six experimental configurations (Figure~\ref{fig:ri-metrics}; Table~\ref{tab:t10-summary}). On Default/GBDT, it reaches accuracy of 0.858, precision of 0.887, and recall of 0.537 at $t{=}10$, substantially higher than all other strategies. Equivalently, by Equation~\ref{eq:ko-precision}, its Kickout---our measure of rejection quality---rises to 0.774 by $t=10$ on Default/GBDT (Figure~\ref{fig:default-rate-kickout}b), the highest of any strategy.

\begin{figure}[t]
    \centering
    \includegraphics[width=\linewidth]{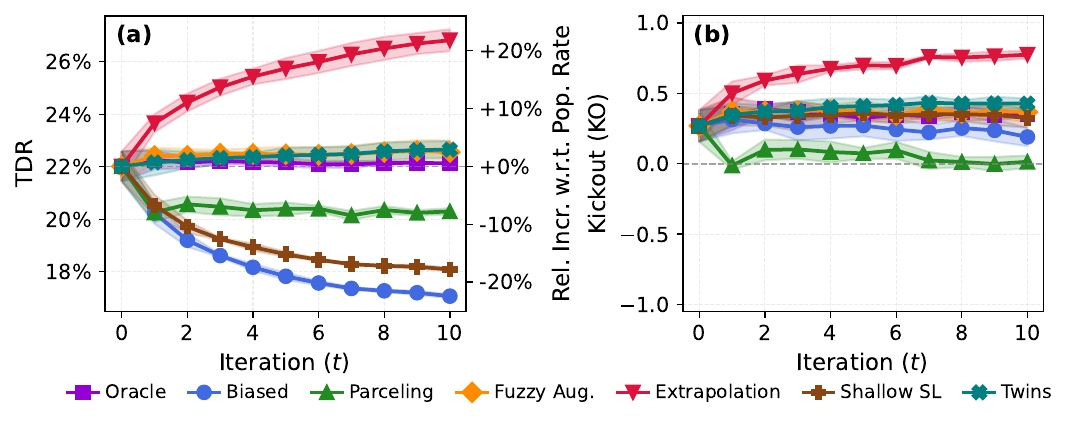}
    \caption{Training set default rate (a) and Kickout (b) over ten iterations (Default, GBDT). The dashed line indicates the population default rate (22.1\%). Extrapolation inflates the default rate by ${\sim}5\%$ above the population rate while Biased deflates it by ${\sim}5\%$---symmetric distortions with opposite mechanisms. %
    }
    \label{fig:default-rate-kickout}
\end{figure}

However, the training set default rate reveals the mechanism behind this apparent superiority (Figure~\ref{fig:default-rate-kickout}a). Extrapolation inflates the training set default rate well above the population rate, from 22.1\% to 26.8\% on Default/GBDT, because it uses the model's own predictions as ground-truth labels for rejected applicants. This creates a self-reinforcing feedback loop: the model predicts that rejected applicants would default, labels them as defaulters, retrains on this augmented data, and becomes even more aggressive in its default predictions. Each iteration reinforces the previous one, and the training set default rate does not return toward the population rate. The distortion does not grow without bound, however: the per-cycle increment decays by an order of magnitude over the ten cycles, from 1.7 percentage points in the first iteration to 0.1 in the last, and TDR flattens toward 26.8\% by $t=10$ (Figure~\ref{fig:default-rate-kickout}a), consistent with the loop approaching a fixed point rather than diverging.

\noindent\textbf{Middle tier: Fuzzy Augmentation and Twins.}
These two strategies maintain accuracy comparable to the Biased baseline (approximately 0.82 on Default/GBDT, against 0.816 for Oracle and 0.833 for Biased) while providing meaningfully better recall, around 0.38 at $t{=}10$ compared to 0.11 for Biased. Their training set default rates stay much closer to the population rate. Fuzzy Augmentation is the most stable: its default rate deviates by less than 1\% from the population rate in every configuration, making it the closest to Oracle on this measure in five of six. Twins shows strong Kickout performance on the Default dataset (reaching approximately 0.43), indicating high-quality rejection decisions.

\noindent\textbf{Weak tier: Parceling and Shallow Self-Learning.}
These strategies provide only marginal improvements over the Biased baseline. Parceling achieves recall of 0.21 on Default/GBDT, better than Biased (0.11) but far below Fuzzy Augmentation (0.39) or Twins (0.38). Its precision drops to approximately 0.50, and its Kickout values hover near zero, occasionally turning negative (notably on PPDai/RF), meaning it sometimes rejects more non-defaulters than defaulters. 
Relative to the Kickout reference point of Section~\ref{sec:methods_data}, this still represents rejections carrying the default label at more than twice the population rate: Parceling's screening is weak compared with the other strategies, but still informative.

Parceling's imputed labels are derived from observed default rates among accepted applicants, and those rates are themselves deflated by survival bias. The bin structure provides little protection against this: whichever parcel supplies a reject's estimate, that estimate is a default rate among accepted applicants. The exploration experiments of Section~\ref{sec:exploration} quantify the reference rate Parceling is missing: under random selection among rejects, roughly 55\%--66\% of those whose outcomes are observed default (Table~\ref{tab:exploration-default-rates}), whereas for the Biased baseline, which trains only on accepted applicants, the default rate among accepts falls to 17.1\% against a population rate of 22.1\% (Figure~\ref{fig:default-rate-kickout}a). 
Parceling nonetheless improves on Biased, because it retains every rejected applicant and so restores the sampling side of the problem in full: its training set is the whole of $\mathcal{D}$ at $t = 10$, where Biased's holds 26,934 tuples. What remains is a labeling shortfall. Within each parcel, rejects are labeled as defaults at the rate observed among the accepts in that parcel, and those rates sit far below the rate among rejects---17.1\% among accepts at $t = 10$, against the roughly 55\%--66\% observed among explored rejects---so the labels Parceling supplies are too often non-default and its TDR settles below the population rate, at 20.3\% by $t = 10$. The resulting shortfall of 1.8 percentage points is about a third of the Biased baseline's deflation of 5.1, so Parceling removes roughly two thirds of it while remaining anchored to accept-side rates.

Shallow Self-Learning achieves recall of only 0.15 on Default/GBDT, and its training set default rate drifts downward similarly to the Biased baseline (dropping to approximately 18\%). The reason is that it incorporates almost none of the rejected applicants: at $t=10$ its training set holds 27{,}585 tuples against 26{,}934 for Biased and 30{,}000 for the strategies that retain every applicant, so it closes only 21\% of the gap between discarding and retaining rejects. 
The novelty filter cannot account for this, since it trims a fixed fraction by construction; the constraint is the labeling rule, which reaches one fifth of the retained rejects in the first round and few thereafter.

\noindent\textbf{Training set default rate as a diagnostic.}
The TDR reveals a striking pattern invisible to standard metrics (Figure~\ref{fig:default-rate-kickout}a). 
By $t=10$ on Default/GBDT, Extrapolation inflates the TDR to 26.8\% ($+21.3\%$ relative to the population rate), while Biased deflates it to 17.1\% ($-22.8\%$ relative): distortions of comparable magnitude in opposite directions.
They arise differently. Biased discards rejected applicants, so its training set is a subpopulation shaped by the model's own decisions and its label rate falls with it. Extrapolation retains every applicant, so the applicants in its training set remain a uniform sample of $D$ and the distortion is confined to the labels it assigns to the rejected ones.
Where Biased makes a sampling error, Extrapolation makes a labeling one; it does not mitigate survival bias so much as relocate it.
Together with the fact that its imputations carry no information the model did not already have, this casts doubt on whether Extrapolation should be considered a reject inference strategy at all, at least for deterministic models under a fixed threshold; variants that sample from the predicted distribution, or that vary the threshold at imputation time, would not share this property. Fuzzy Augmentation is precisely that variant: it draws each reject's label from $\mathrm{Bernoulli}(\hat{p}_i)$ instead of thresholding $\hat{p}_i$ at $\tau$, and its training set default rate tracks the population (22.6\% against 22.1\%) where Extrapolation reaches 26.8\%.

Among the remaining strategies, Parceling (TDR 20.3\%, $-8.2\%$ relative) and Shallow Self-Learning (18.1\%, $-18.2\%$) offer only partial corrections. In contrast, Fuzzy Augmentation (22.6\%, $+2.1\%$) and Twins (22.7\%, $+2.5\%$) track the population rate nearly as closely as Oracle (22.1\%). This is the property standard metrics fail to capture, but it is a weaker guarantee than it appears: the TDR constrains only the rate at which the label 1 appears, not which applicants receive it. A strategy that labeled rejects at the right rate but the wrong applicants would score identically. Parceling illustrates the point from the other side: its TDR is closer to the population rate than Shallow Self-Learning's, while the standard metrics rank it last of the seven. We do not read this as evidence that either ordering is correct---the standard metrics are subject to the circularity described above, and the TDR constrains only the label marginal---but the two do not measure the same thing, and neither on its own certifies that a strategy has recovered the missing labels. Section~\ref{sec:exploration} takes the alternative route of observing a sample of the missing outcomes directly.

\begin{table}[t]
    \centering
    \caption{Accuracy and recall at $t=10$ for all strategies across all six configurations, averaged over five runs. Bold indicates the best value per column (excluding Oracle). Extrapolation ranks highest on both metrics in every configuration; part of its evaluation labels are its own imputations (Section~\ref{subsec:ri-comparison}). The Gap row reports the accuracy difference between Extrapolation and Oracle.}
    \label{tab:t10-summary}
    \small
    \setlength{\tabcolsep}{4pt}
    \begin{tabular}{l cc cc cc}
        \toprule
        & \multicolumn{2}{c}{Default} & \multicolumn{2}{c}{PPDai} & \multicolumn{2}{c}{LendingClub} \\
        \cmidrule(lr){2-3} \cmidrule(lr){4-5} \cmidrule(lr){6-7}
        Strategy & GBDT & RF & GBDT & RF & GBDT & RF \\
        \midrule
        \multicolumn{7}{l}{\textit{Accuracy}} \\
        \addlinespace
        Biased          & 0.833 & 0.833 & 0.785 & 0.786 & 0.783 & 0.782 \\
        Oracle          & 0.816 & 0.812 & 0.780 & 0.780 & 0.779 & 0.777 \\
        Parceling       & 0.796 & 0.794 & 0.775 & 0.777 & 0.778 & 0.777 \\
        Fuzzy Aug.      & 0.822 & 0.815 & 0.782 & 0.779 & 0.779 & 0.776 \\
        Extrapolation   & \textbf{0.858} & \textbf{0.863} & \textbf{0.796} & \textbf{0.801} & \textbf{0.796} & \textbf{0.790} \\
        Shallow SL      & 0.835 & 0.838 & 0.788 & 0.791 & 0.784 & 0.783 \\
        Twins           & 0.822 & 0.827 & 0.778 & 0.780 & 0.778 & 0.779 \\
        \addlinespace
        \rowcolor{gray!10}
        \textit{Gap (Extrap.--Oracle)} & \textit{+.042} & \textit{+.051} & \textit{+.016} & \textit{+.021} & \textit{+.017} & \textit{+.013} \\
        \addlinespace
        \midrule
        \multicolumn{7}{l}{\textit{Recall}} \\
        \addlinespace
        Biased          & 0.111 & 0.150 & 0.036 & 0.057 & 0.019 & 0.035 \\
        Oracle          & 0.367 & 0.371 & 0.102 & 0.125 & 0.067 & 0.061 \\
        Parceling       & 0.210 & 0.217 & 0.048 & 0.082 & 0.035 & 0.039 \\
        Fuzzy Aug.      & 0.386 & 0.404 & 0.107 & 0.138 & 0.065 & 0.063 \\
        Extrapolation   & \textbf{0.537} & \textbf{0.619} & \textbf{0.299} & \textbf{0.349} & \textbf{0.286} & \textbf{0.236} \\
        Shallow SL      & 0.150 & 0.189 & 0.047 & 0.075 & 0.040 & 0.048 \\
        Twins           & 0.376 & 0.418 & 0.048 & 0.098 & 0.032 & 0.045 \\
        \bottomrule
    \end{tabular}
    \vspace*{-0.4cm}
\end{table}

A practitioner evaluating strategies by accuracy, precision or recall would rationally select Simple Extrapolation as the best strategy, and Kickout, being a rescaling of precision, gives the same answer. Evaluating against imputed labels is not in itself objectionable---reject inference exists precisely to supply them, and an accurate imputation can be a reasonable proxy for the outcome that was never observed. What distinguishes Extrapolation is that its imputed labels are not independent of the predictions being scored, and the mechanism is worth stating carefully. 
At cycle $t$, the model $M_{t-1}$ screens the arriving applicants and Extrapolation records $y = 1$ for every applicant it rejects. \footnote{Since both RF and GBDT are deterministic at inference time and the threshold is fixed, re-applying $M_{t-1}$ to rejected applicants---who received $\hat{y}=1$ in the first prediction---yields $\hat{y}=1$ again, so all rejected applicants receive the imputed label $y=1$.} 
These rows are joined with $\mathcal{D}_{t-1}$ to form $\mathcal{D}_t$, on which $M_t$ is trained and evaluated against a held-out split of it.
The label and the prediction thus come from successive models rather than from one. The distortion does not depend on those two models agreeing, however, but only on the labels being uniformly $y=1$: a false positive requires a row labeled $y=0$, so no imputed tuple can ever become one, whatever the retrained model predicts. Each is either a true positive or a false negative. These tuples can only add to $TP$, never to $FP$, and so push precision ($TP/(TP+FP)$) upward as they accumulate. 
Because Extrapolation retains every applicant, $\mathcal{D}_{10}$ is the whole of $D$ and therefore contains $30{,}000\cdot\pi = 6{,}636$ true defaulters, against 8{,}046 tuples labeled as defaults. 
Thus, $8{,}046-6{,}636=1{,}410$ applicants who did not default carry the recorded label $y=1$, 17.5\% of those 8{,}046. The evaluation split is a uniform 20\% of $\mathcal{D}_{10}$ and inherits this composition: of its roughly $8{,}046\cdot0.2\approx1{,}609$ rows with $y=1$, $6{,}636\cdot0.2\approx1{,}327$ are true defaulters and the remaining 282 are not. An imputed tuple whose applicant did default carries $y = 1$ under observed outcomes too, and falls in the same cell of the confusion matrix either way. Only the 282 change what is counted: each would carry $y = 0$ under observed outcomes, so a prediction of $\hat{y} = 1$ on one of them would be a false positive rather than the true positive Extrapolation records.
Its precision reaches 0.887, where no other strategy exceeds 0.72.
No other reject inference strategy assigns the same imputed label to every reject, though they differ in how far they escape the circularity. Parceling imputes from the observed bad rates among accepted applicants, and Fuzzy Augmentation from a stochastic draw on the predicted probability, which can contradict the prediction; in both cases a reject's label is not fixed by the model's own output on that reject. Twins and Shallow Self-Learning draw on the labels recorded in $\mathcal{D}_t$, which from the second cycle onward include their own imputations from earlier cycles, so their labels are not independent of their own history---but neither reproduces the current model's predictions, and a prediction agreeing with one of their labels has still been checked against something other than itself. The four metrics (accuracy, precision, recall and Kickout) are computed over the imputed tuples together with the observed ones. 
The training set default rate is the one quantity that escapes this circularity, since it is measured against a population rate the model cannot influence. It shows Extrapolation's distortion plainly: its TDR
rises from 22.1\% to 26.8\% over the ten iterations (Figure~\ref{fig:default-rate-kickout}a).

\noindent\textbf{The Oracle Paradox.} The most revealing finding emerges from the ranking itself, and it is already visible in the comparison between the two baselines. 
Biased attains higher accuracy than Oracle in every configuration, by 1.8\% and 2.1\%  on Default, 0.5\% and 0.6\% on PPDai, and 0.4\% and 0.5\% on LendingClub (0.833 vs.\ 0.816 on Default/GBDT; Table~\ref{tab:t10-summary}), even though both are evaluated entirely
against observed labels and Oracle has strictly more information. 
The margin is not an artifact of noise: it is positive at $t = 10$ in all six configurations and at every intermediate cycle in four of them, the exceptions being two cycles on PPDai/GBDT and one on PPDai/RF; the standard deviation bands at $t = 10$ separate in five of six, overlapping only on PPDai/GBDT (Figure~\ref{fig:biased-oracle} shows Default/GBDT). Withholding the outcomes of rejected applicants raises measured accuracy.
Since the two baselines are scored on the same terms, Biased's advantage cannot be attributed to differences in what each is evaluated against; it follows from the evaluation population, for the reasons given in Section~\ref{subsec:baselines}.
Extrapolation ranks above Oracle on accuracy, precision and recall in all six configurations (Table~\ref{tab:t10-summary}), which is consistent with the same pattern but is not a comparison on equal terms. As noted in Section~\ref{sec:methods_data}, part of Extrapolation's evaluation split carries its own imputed labels, and on those rows its predictions agree with the recorded labels by construction. We therefore report the ranking without attributing the margin to predictive quality: how much of it would survive an evaluation against observed outcomes alone is not determined by these experiments, and we return to this in Section~\ref{sec:conclusions}.

This finding extends beyond the reject inference comparison. It demonstrates that standard classification metrics evaluated on the observable population are structurally inadequate for assessing model quality in the presence of survival bias. A practitioner selecting a reject inference strategy by standard evaluation metrics would choose Extrapolation, the strategy that most aggressively amplifies survival bias, and would conclude that their model is improving with each retraining cycle.

The magnitude of these effects varies with dataset scale. On LendingClub ($n_0{=}60{,}000$), all strategies exhibit the same qualitative patterns but with compressed magnitudes: accuracy differences between strategies narrow 
to ${\sim}2\%$ and recall degradation is less severe (Figure~\ref{fig:lendingclub-ri}, Appendix~\ref{app:ri-dataset-size}). This scale effect has practical significance. Smaller lenders and new product lines, where initial training data is more limited, are most vulnerable to the survival bias spiral documented here.

\section{Controlled Exploration as Bias Mitigation}
\label{sec:exploration}

The reject inference strategies evaluated in Section~\ref{sec:reject-inference} all attempt to \emph{impute} labels for rejected applicants using statistical assumptions about the missing data. Each introduces its own form of model-dependent bias, and none can recover the true outcome distribution without assumptions that are fundamentally untestable under selection.
We take a different approach inspired by the \emph{exploration--exploitation} tradeoff in sequential decision-making~\cite{lattimore2020bandit} and A/B or split testing. Rather than guessing at missing labels, the lender deliberately approves a fraction $r$ of rejected applicants and observes their true repayment outcomes. This \emph{controlled exploration}, governed by the parameter $r$, breaks the feedback loop at a measurable cost: the defaults incurred by explored applicants who fail to repay. The parameter $r$ interpolates between the two baselines from Section~\ref{subsec:baselines}: $r = 0$ corresponds to Biased (maximum exploitation) and $r = 1$ to Oracle (maximum exploration). The central question is: \emph{what is the cost of moving from one extreme to the other?}

Given a trained model at iteration $t$, let $\mathcal{R}_t$ denote the set of rejected applicants. The exploration strategy selects a subset $\mathcal{E}_t \subseteq \mathcal{R}_t$ with $|\mathcal{E}_t| = \lceil r \cdot |\mathcal{R}_t| \rceil$ for approval, then augments the training data with these applicants and their \emph{observed} (not imputed) labels. The selection of which rejects to explore is governed by a \textbf{sampling strategy}: \textbf{Least Risky} selects the $r$-fraction with lowest $P(\text{default})$, aiming to minimize cost by choosing the rejects closest to the decision boundary; \textbf{Most Risky} selects the $r$-fraction with highest $P(\text{default})$, probing the region where the model has the least information; and \textbf{Random} selects uniformly, providing an unbiased sample that requires no model scores.\footnote{An uncertainty-based strategy (minimizing $|P(\text{default}) - \tau|$) is equivalent to Least Risky when $\tau = 0.5$, since all rejects satisfy $P(\text{default}) > \tau$ and the two orderings coincide.} We vary $r \in \{0, 0.01, 0.02, 0.05, 0.10, 0.15, 0.20, 0.30, 0.50, 0.75, 1.0\}$ across all three datasets and both modeling techniques, for the three selection mechanisms with five runs per configuration. For legibility, the figures plot the subset $r \in \{0, 0.05, 0.10, 0.20, 0.50, 1.0\}$; all reported values use the full range.

\begin{figure}[t]
  \centering
  \begin{subfigure}[b]{0.49\linewidth}
    \centering
    \includegraphics[width=\linewidth]{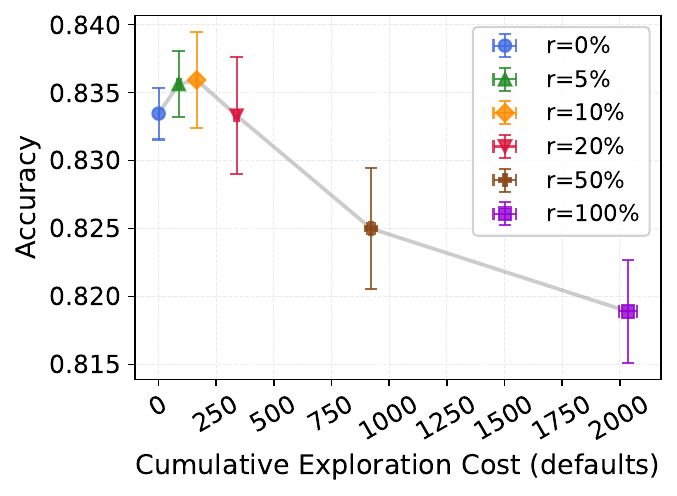}
    \caption{Accuracy vs.\ cumulative cost.}
    \label{fig:pareto-accuracy}
  \end{subfigure}
  \hfill
  \begin{subfigure}[b]{0.48\linewidth}
    \centering
    \includegraphics[width=\linewidth]{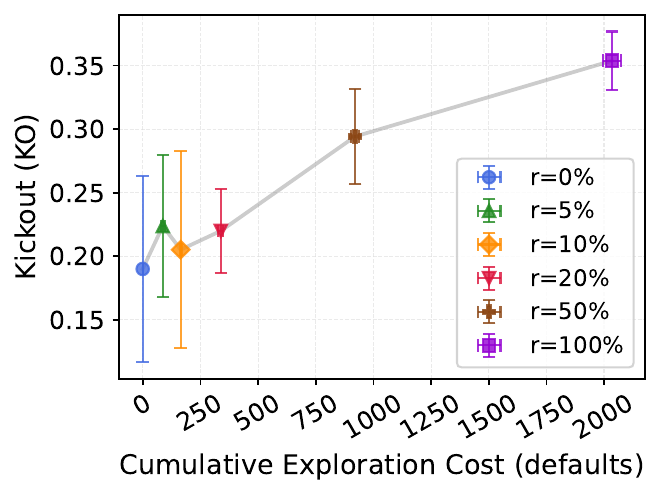}
    \caption{Kickout vs.\ cumulative cost.}
    \label{fig:pareto-kickout}
  \end{subfigure}
  \caption{Cost--benefit tradeoff on the Default dataset (GBDT, Least Risky sampling). 
  At full exploration ($r=1.0$), accuracy drops from $0.833$ to $0.819$ ($-1.5\%$) while Kickout rises from $0.19$ to $0.35$ ($+0.16$).
  Error bars indicate $\pm 1$ standard deviation across 5 runs.}
  \label{fig:exploration-tradeoff}
\end{figure}

\subsection{The Cost--Benefit Tradeoff}
\label{sec:exploration-tradeoff}

Figure~\ref{fig:exploration-tradeoff} presents the core finding. Accuracy and Kickout move in opposite directions as exploration increases: on Default/GBDT under Least Risky sampling, accuracy drops from $0.833$ at $r = 0$ to $0.819$ at $r = 1.0$ ($-1.5\%$), while Kickout rises from $0.19$ to $0.35$ ($+0.16$).
The paths between these endpoints are not monotone. Accuracy rises slightly at the smallest exploration rates before declining, and Kickout is non-monotone at intermediate rates ($r\leq 0.3$); both endpoint differences, by contrast, hold in every run and under all three sampling rules.
Recall that differences are absolute unless stated as relative.
Accuracy is essentially flat for $r \le 0.2$ and declines only at higher rates: the accumulated pool is shaped by the model's own selection decisions, and small amounts of exploration barely change it. As exploration shifts the training data toward the true population, the evaluation becomes harder but the model's rejection decisions become more aligned with actual risk.

\begin{table}[t]
\centering
\caption{Accuracy drop and Kickout at full exploration ($r = 1.0$, Random sampling) relative to the biased baseline ($r = 0$), measured at the final iteration ($t = 10$) and averaged over runs. $\Delta$Acc and TDR gap are absolute differences; the KO column reports the Kickout value at $r = 1.0$. Relative TDR distortion is the TDR gap divided by the biased TDR.}
\label{tab:exploration-summary}
\small
\begin{tabular}{@{}llcccc@{}}
\toprule
Dataset & Model & $\Delta$Acc (\%) & KO at $r{=}1.0$ & TDR Gap (\%) & Rel.\ TDR Distortion \\
\midrule
Default     & GBDT & $-1.50$ & $0.344$ & $5.05$ & $29.6\%$ \\
Default     & RF   & $-1.91$ & $0.312$ & $5.09$ & $29.9\%$ \\
PPDai       & GBDT & $-0.53$ & $0.232$ & $1.17$ & $5.3\%$ \\
PPDai       & RF   & $-0.49$ & $0.190$ & $1.33$ & $6.2\%$ \\
LendingClub & GBDT & $-0.33$ & $0.150$ & $0.64$ & $2.9\%$ \\
LendingClub & RF   & $-0.35$ & $0.081$ & $0.57$ & $2.6\%$ \\
\bottomrule
\end{tabular}
\end{table}

This pattern is consistent in direction but varies in magnitude across configurations, measured under Random sampling (Table~\ref{tab:exploration-summary}); comparisons for all other configurations are provided in Appendix~\ref{app:exploration}. Default exhibits the strongest effect (accuracy drops of $1.5\%$--$1.9\%$, Kickout gains of $0.15$--$0.23$), PPDai shows intermediate sensitivity ($0.49\%$--$0.53\%$ accuracy decline, $0.10$--$0.12$ Kickout gain), and LendingClub is the most attenuated ($0.33\%$--$0.35\%$ accuracy decline, with no systematic Kickout gain at any exploration rate). The TDR gap between $r = 0$ and $r = 1.0$ at $t=10$ serves as a diagnostic for feedback loop severity: on Default, biased training makes the borrower population appear approximately 30\% safer than reality, relative to the biased TDR the lender observes ($\text{TDR} = 17.1\%$ vs.\ $22.1\%$; the relative distortion as defined in Table~\ref{tab:exploration-summary}), while on LendingClub the relative distortion is only $2.6\%$--$2.9\%$.

\subsection{The Diagnostic Zone}
\label{sec:diagnostic-zone}

A lender considering exploration faces a practical question: how much is enough? Even minimal exploration produces detectable shifts. At $r = 0.05$, accuracy changes are negligible across all configurations ($|\Delta| < 0.5\%$), while TDR already moves toward the population default rate on every configuration (Figure~\ref{fig:temporal-multi}).
Kickout responds at this rate on Default and on PPDai/GBDT, but is erratic across rates on PPDai/RF and on LendingClub it fluctuates within 0.04 (GBDT) and 0.06 (RF) of its baseline without systematic improvement at any rate. TDR is therefore the more reliable diagnostic signal in the low-$r$ regime and on the largest dataset.

The TDR shift is a reliable signal that a feedback loop exists, but not a measure of how severe it is: at $r=0.05$ it moves the TDR by between 0.03 and 0.35 percentage points across the six configurations, against TDR gaps between $r=0$ and $r=1.0$ of 0.6 to 5.1 points (Table~\ref{tab:exploration-summary}). The explored rejects themselves give the magnitude. The proof of Lemma~\ref{lem:cleaning} decomposes the population rate as $\pi = \alpha\,\mathrm{prec} + (1-\alpha)\beta$, where $\alpha = P_{\mathcal{D}}(\hat{y}=1)$ is the fraction of arriving applicants the model rejects, $\mathrm{prec}$ the default rate among rejected applicants and $\beta$ that among accepted ones (we drop the cycle index here). The deflation the loop converges to is therefore $\pi - \beta = \alpha(\mathrm{prec} - \beta)$: the rejection rate times the excess default rate among rejected applicants. The lender observes $\alpha$, estimates $\mathrm{prec}$ by the default rate $\hat{d}$ among the rejects it approved for exploration, and uses its own $\mathrm{TDR}_t$ in place of $\beta$, giving $\alpha(\hat{d} - \mathrm{TDR}_t)$---an estimate requiring no knowledge of $\pi$. Here $\hat{d}$ is the default rate among all rejects explored over the course of the run, $\alpha$ the mean rejection rate across its cycles, and $\mathrm{TDR}_t$ the value at the final cycle.
The substitution is exact only in the limit: $\mathrm{TDR}_t$ is the default rate of the accumulated pool, which still carries the weight of the unbiased initial sample $D_0$ and so sits above $\beta$, the rate among a single accepted batch.
The estimates therefore tend to run slightly high, for the same reason that the gap at $t = 10$ has not yet reached the value $\alpha(\mathrm{prec} - \beta)$ the loop converges to.
At $r=0.02$ the estimate yields 6.1\% and 5.8\% on the two Default configurations, 1.1\% and 1.4\% on PPDai, and 0.8\% and 0.7\% on LendingClub, against TDR gaps of 5.1\%, 5.1\%, 1.2\%, 1.3\%, 0.6\% and 0.6\%, respectively---above the gap in five of the six configurations and below it, by 0.1\%, on PPDai/GBDT. 
They are nevertheless accurate enough to separate the three regimes.

On Default, accuracy does not decline by more than $0.5\%$ until $r = 0.30$ (Least Risky, both models) or $r = 0.50$--$0.75$ (Most Risky and Random, depending on the model); on LendingClub the threshold is never reached, and on PPDai it is reached only at full exploration, in two of the six dataset--model--strategy combinations (Table~\ref{tab:sweet-spot}).

\begin{figure}[t]
  \centering
  \includegraphics[width=\linewidth]{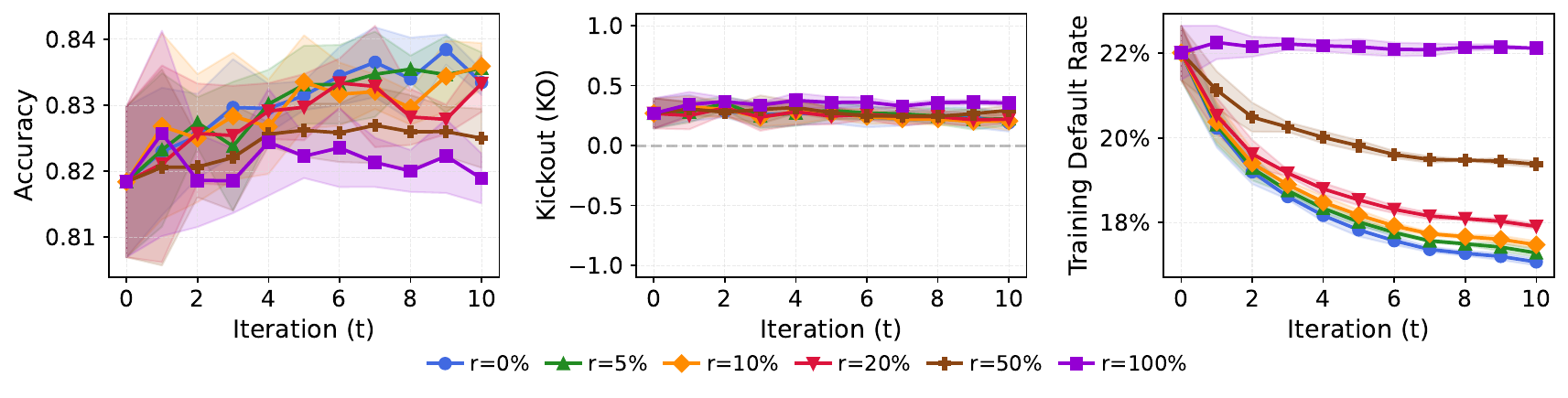}
  \caption{Temporal evolution of accuracy, Kickout, and TDR across iterations on the Default dataset (GBDT, Least Risky sampling). Higher exploration rates shift TDR toward the population default rate (${\sim}22\%$); accuracy is lower and Kickout higher at the highest rates than at $r=0$, consistent with learning from a more representative training distribution.
  }
  \label{fig:temporal-multi}
\end{figure}

\begin{table}[t]
\centering
\caption{Smallest exploration rate in our grid at which accuracy declines by more than $0.5\%$ in absolute terms below the biased baseline ($r = 0$). Accuracy is evaluated at the final iteration ($t = 10$) and averaged over five runs. ``Never'' indicates that even at $r = 1.0$ this threshold is not crossed.}
\label{tab:sweet-spot}
\small
\begin{tabular}{@{}llccc@{}}
\toprule
Dataset & Model & Least Risky & Most Risky & Random \\
\midrule
Default     & GBDT & $0.30$ & $0.75$ & $0.75$ \\
Default     & RF   & $0.30$ & $0.50$ & $0.50$ \\
PPDai       & GBDT & never  & never  & $1.00$ \\
PPDai       & RF   & never  & $1.00$ & never \\
LendingClub & GBDT & never  & never  & never \\
LendingClub & RF   & never  & never  & never \\
\bottomrule
\end{tabular}
\end{table}

This defines a \emph{diagnostic zone} at $r \in [0.02, 0.05]$: exploration rates that reveal the presence and severity of survival bias at near-zero cost. In practice, a bank could run a controlled experiment at $r = 0.02$--$0.05$ purely as a diagnostic tool, without committing to permanent policy changes. On the Default dataset, $r = 0.02$ incurs 43--45 defaults across all iterations under Random sampling (36--37 under Least Risky, 54--63 under Most Risky), but the information gained is substantial: the bank learns whether its model is operating on a representative sample or on a fiction of its own creation.

\begin{figure}[t]
  \centering
  \begin{subfigure}[b]{0.48\linewidth}
    \centering
    \includegraphics[width=\linewidth]{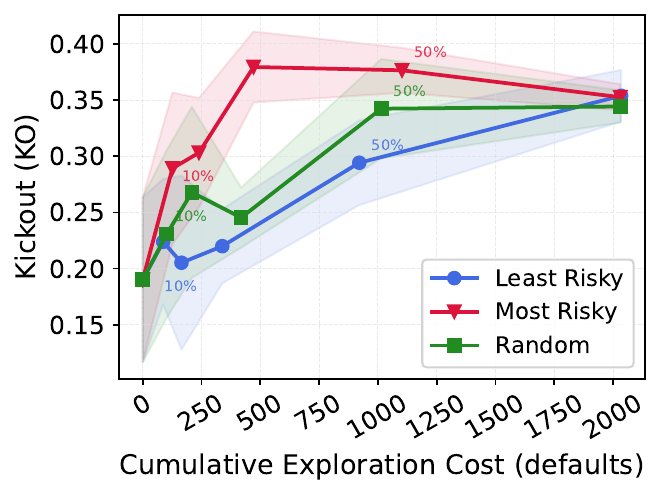}
    \caption{Kickout vs.\ cumulative cost.}
    \label{fig:ranking-kickout}
  \end{subfigure}
  \hfill
  \begin{subfigure}[b]{0.48\linewidth}
    \centering
    \includegraphics[width=\linewidth]{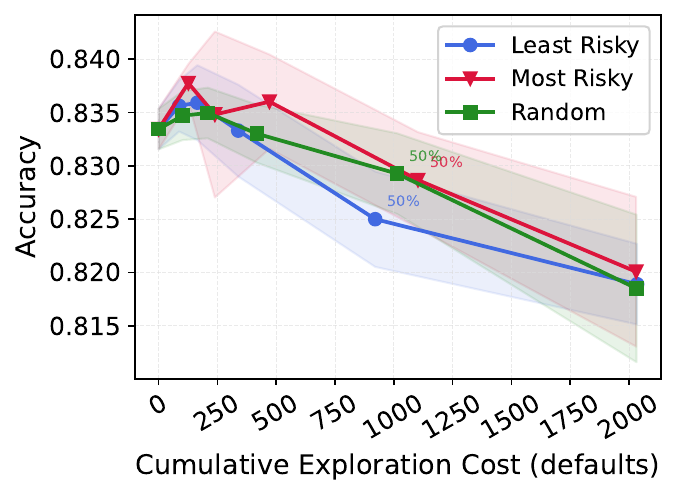}
    \caption{Accuracy vs.\ cumulative cost.}
    \label{fig:ranking-accuracy}
  \end{subfigure}
  \caption{Sampling strategy comparison on the Default dataset (GBDT). Most Risky achieves the highest Kickout at every intermediate exploration rate and preserves accuracy better than Least Risky despite higher per-applicant cost. Random is a robust middle ground.}
  \label{fig:ranking-comparison}
\end{figure}

\subsection{Sampling Strategy Comparison}
\label{sec:ranking-comparison}

Figure~\ref{fig:ranking-comparison} compares the three sampling strategies. On Default/GBDT, Most Risky exploration produces the highest Kickout at every exploration rate in our grid with $0.05\leq r\leq 0.75$,\footnote{At $r=0$ and $r=1$ the strategies select the same applicants by construction, since no rejects and all rejects are explored respectively. At $r=1$ the reported values still differ slightly, because the three rules append the same applicants in different orders and the train/test split depends on that order; the spread (0.0016 in accuracy, 0.010 in Kickout) indicates the scale of this variation.} by margins of $0.034$--$0.134$ over Random: at $r = 0.20$, it achieves $0.379$ compared to $0.245$ (Random) and $0.220$ (Least Risky), a difference of $0.134$.
This is intuitive: the rejects the model considers highest-risk are the most underrepresented in biased training data, and observing their true outcomes provides the strongest learning signal. Most Risky also preserves accuracy better than expected, maintaining positive accuracy deltas through $r = 0.20$ on Default/GBDT. Least Risky is dominated: despite its appeal as the lowest-cost option, it produces the smallest Kickout improvements because exploring ``safe'' rejects provides redundant rather than informative training signal. Random requires no model scores, avoids systematic exploration bias, and achieves Kickout between the other two strategies, making it the practical default for deployment.

\begin{table}[t]
\centering
\caption{Default rate among explored rejects at $r = 0.10$ by sampling strategy, measured at the final iteration ($t = 10$) and averaged over runs. The gradient from ${\sim}50\%$ (Least Risky) to ${\sim}80\%$ (Most Risky) confirms that the model's risk ranking is meaningful even among rejected applicants. All rates far exceed the population default rate (22.1\% on Default, 23.0\% on PPDai, 22.4\% on LendingClub), as expected for applicants the model screened out.}
\label{tab:exploration-default-rates}
\small
\begin{tabular}{@{}llccc@{}}
\toprule
Dataset & Model & Least Risky & Random & Most Risky \\
\midrule
Default     & GBDT & $50.5\%$ & $65.6\%$ & $82.0\%$ \\
Default     & RF   & $52.6\%$ & $57.0\%$ & $82.1\%$ \\
PPDai       & GBDT & $56.8\%$ & $66.5\%$ & $72.3\%$ \\
PPDai       & RF   & $46.4\%$ & $65.4\%$ & $79.2\%$ \\
LendingClub & GBDT & $53.7\%$ & $58.8\%$ & $66.4\%$ \\
LendingClub & RF   & $50.3\%$ & $54.6\%$ & $60.5\%$ \\
\bottomrule
\end{tabular}
\end{table}

The default rates among explored rejects (Table~\ref{tab:exploration-default-rates}) reveal that the model's quality has two distinct dimensions. The gradient from ${\sim}50\%$ (Least Risky) to ${\sim}80\%$ (Most Risky) confirms that the model's \emph{ranking} of applicants by risk is informative: it captures genuine signal about relative default likelihood, even among applicants it was trained to reject. Its probability estimates, by contrast, are accurate in aggregate but too high at the top of that ranking. 
On Default/GBDT, at $r = 0.10$ and averaged over iterations, the mean predicted default probability among rejects selected at random is 68\%, against an observed default rate of 67\%, and the two track each other within a few points at every exploration rate and in every iteration.
Among the rejects the model scores as highest risk, however, it predicts 85\% where 78\% default; Table~\ref{tab:exploration-default-rates} reports the same rate at  $t = 10$ alone, where it is 82.0\%.
The model's ranking of rejected applicants is therefore not what needs fixing. What exploration contributes is observed outcomes in the region where the model's probability estimates are furthest off.
\section{Conclusions}
\label{sec:conclusions}
The cost--benefit tradeoff encapsulates the central tension of our research. A bank that evaluates its scoring model by standard metrics has no reason to explore: accuracy offers no reliable gain at any exploration rate, and declines at the rates high enough to improve rejection quality, so it does not reward exploration, while the defaults incurred are immediate and visible. This is the optimal strategy under the metric the bank is optimizing, and the worst possible strategy for learning.
A short-term diagnostic experiment at rates $r = 0.02$--$0.05$, using random selection among rejected applicants, is sufficient to estimate the severity of the feedback loop, from the default rate among explored rejects. The shift in TDR alone understates the deflation the loop converges to (Section~\ref{sec:diagnostic-zone}). The heterogeneity across datasets reinforces this: Default, where the feedback loop distorts TDR by ${\sim}30\%$ relative to the rate the lender observes, represents a scenario where survival bias is a first-order concern; LendingClub, with only $2.6\%$--$2.9\%$ relative distortion, is a setting where survival bias may be a secondary modeling concern. 
The proposed diagnostic exploration experiment allows the lender to determine which regime it occupies \emph{before} investing in complex reject inference methods or trusting that its current model is performing well.
The reject inference comparison reveals three results with direct implications for credit scoring practice. First, no reject inference strategy fully eliminates survival bias in the temporal setting; even the most effective methods (Fuzzy Augmentation, Twins), which keep the training set default rate near the population rate, only slow the accumulation.
Second, Simple Extrapolation does not mitigate survival bias but relocates it: the Biased baseline deflates the TDR to 17.1\% by discarding rejected applicants, while Extrapolation inflates it to 26.8\% by retaining every applicant and assigning them its own predictions as labels---deviations of comparable magnitude from the population rate of 22.1\%, but a sampling distortion in the first case and a labeling one in the second.
Third, the fact that the Biased baseline outperforms Oracle on accuracy, with both evaluated against observed labels, demonstrates that these metrics are not merely noisy but systematically favor models scored on a population their own decisions have shaped.
Taken together, these findings reinforce the central argument of this paper: iterative model-based lending produces performance indicators that are fundamentally disconnected from the model's actual predictive quality for the full applicant population. 

\revision{\noindent\textbf{Limitations.} Our simulation assumes a stationary applicant distribution; in practice, population drift would interact with the survival bias feedback loop in ways we do not model. The fixed threshold $\tau=0.5$ is a simplifying assumption; real deployments tune this threshold to reflect the asymmetric costs of approving defaulters versus rejecting creditworthy applicants, and threshold choice would affect the rate of feedback loop accumulation. 
Our evaluation protocol measures what a lender performing reject inference could observe: each strategy is scored on a held-out split of its own accumulated training set, against the labels in that set.
For the five reject inference strategies these are observed outcomes for the accepted applicants and imputed labels for the rejected applicants they incorporate. The two baselines are exempt, since neither imputes any label: under Biased every label is an observed outcome for an accepted applicant, and under Oracle for every applicant.
This is the right frame for the question we ask, which is what a practitioner would conclude from their own data, but it does mean the five reject inference strategies are not scored on common terms, and comparisons between them should be read accordingly. The inversion between Biased and Oracle reported in Section~\ref{subsec:ri-comparison}, being a comparison on equal terms between the two exempt baselines, does not inherit this caveat.
It also means the reported trajectories describe the accumulated pool rather than the arriving applicants. These are different populations, and metrics on one need not track metrics on the other. A complementary evaluation on each incoming batch, against those applicants' true outcomes, would give the rejection rate and rejection precision on the applicant population directly, at the cost of one additional prediction per cycle; the two evaluations answer different questions, and we leave this to future work. In particular, how much of Extrapolation's margin over Oracle would survive an evaluation against observed outcomes alone is not determined by these experiments.

Our findings are established for tree-based ensemble methods (GBDT and RF); models with different inductive biases may exhibit different magnitudes, though we expect the structural failure mode to persist as it follows from survival bias rather than from any property specific to tree-based models. Finally, controlled exploration, while assumption-free in principle, faces deployment constraints in practice: deliberately approving a fraction of rejected applicants raises regulatory and ethical considerations that would need to be addressed before implementation.}

In future work, we plan to quantify the social and economic costs of survival bias, and to study if and how the feedback loop documented in this work creates demographic disparities in the rejected applicant pool.

\bibliographystyle{splncs04}
\bibliography{bibliography_aies25,bibliography_ecml26}

\newpage
\appendix

\section{Data Preprocessing}
\phantomsection
\label{app:preprocessing}

For the \textbf{Default} dataset, no additional preprocessing was required beyond standard encoding of categorical features.

For the \textbf{PPDai} dataset, we down-sampled the majority class from the original dataset (55,596 rows, 87.1\% non-default) to 31,389 rows with a 77.0\% imbalance ratio, approximately matching the Default dataset's class distribution, of which 31,230 remain after
dropping rows with missing values.

For the \textbf{LendingClub} dataset, preprocessing involved several stages. We first filtered to loans originated in 2017, then discarded the 10 columns with the highest null counts (all exceeding 314,000 nulls) and dropped remaining rows containing nulls. We constructed a binary default label from the \texttt{loan\_status} field, treating ``Default,'' ``Charged Off,'' ``Late (31--120 days),'' ``Late (16--30 days),'' and ``Does not meet the credit policy. Status:Charged Off'' as positive (default) cases. We then removed 40 columns containing post-hoc information that would constitute data leakage, including payment outcomes (\texttt{total\_pymnt}, \texttt{total\_rec\_int}, \texttt{last\_pymnt\_amnt}), recovery fields (\texttt{recoveries}, \texttt{collection\_recovery\_fee}), post-origination credit pulls (\texttt{last\_fico\_range\_high/low}), current delinquency indicators (\texttt{acc\_now\_delinq}, \texttt{num\_tl\_30dpd}), and hardship/settlement flags. The resulting dataset was down-sampled to 302,595 rows with a 77\% imbalance ratio; dropping rows with missing values leaves the 279,520 rows used in the experiments, at a 77.6\% imbalance ratio. 

For all datasets, remaining rows with missing values were dropped, categorical features were encoded as integer codes, and all features
were standardized using \texttt{StandardScaler} prior to model training.

\section{Model Selection}
\phantomsection
\label{app:model-selection}

We evaluated five binary classifiers on all three datasets using 80/20 train-test splits, averaging performance over 10 runs. Tables~\ref{tab:app-perf-default} and~\ref{tab:app-perf-ppdai-lc} report accuracy, precision, and recall for each model. Based on consistent top-tier performance across datasets, we selected Gradient Boosted Decision Trees (GBDT) and Random Forest (RF) for all subsequent experiments.

\begin{table}[ht]
    \centering
    \caption{Model performance on the Default dataset.}
    \label{tab:app-perf-default}
    \small
    \begin{tabular}{lccc}
        \toprule
        Model & Accuracy & Precision & Recall \\
        \midrule
        GBDT                & 0.820 & 0.672 & 0.365 \\
        Random Forest       & 0.814 & 0.636 & 0.369 \\
        Logistic Regression & 0.809 & 0.713 & 0.231 \\
        Decision Tree       & 0.807 & 0.605 & 0.371 \\
        MLP                 & 0.796 & 0.558 & 0.369 \\
        \bottomrule
    \end{tabular}
\end{table}

\begin{table}[ht]
    \centering
    \caption{Model performance on PPDai and LendingClub.}
    \label{tab:app-perf-ppdai-lc}
    \small
    \begin{tabular}{lccc}
        \toprule
        \multicolumn{4}{l}{\textit{PPDai}} \\
        \addlinespace
        Model & Accuracy & Precision & Recall \\
        \midrule
        GBDT                & 0.779 & 0.627 & 0.099 \\
        Random Forest       & 0.778 & 0.587 & 0.120 \\
        Logistic Regression & 0.773 & 0.578 & 0.047 \\
        Decision Tree       & 0.767 & 0.479 & 0.124 \\
        MLP                 & 0.723 & 0.351 & 0.241 \\
        \addlinespace
        \midrule
        \multicolumn{4}{l}{\textit{LendingClub}} \\
        \addlinespace
        Model & Accuracy & Precision & Recall \\
        \midrule
        GBDT                & 0.780 & 0.578 & 0.067 \\
        Random Forest       & 0.778 & 0.548 & 0.062 \\
        Logistic Regression & 0.779 & 0.536 & 0.092 \\
        MLP                 & 0.777 & 0.528 & 0.058 \\
        Decision Tree       & 0.772 & 0.452 & 0.088 \\
        \bottomrule
    \end{tabular}
\end{table}

\section{Initial Training Set Size}
\phantomsection
\label{app:nzero}

We determine the initial training set size $n_0$ via learning curve analysis. For each dataset, we train models at increasing sample sizes $\{k, 2k, \dots, n\}$ with step size $k$ and identify the point at which accuracy stabilizes. Figure~\ref{fig:app-n0-default} illustrates this procedure for the Default dataset. This yields $n_0 = 5{,}000$ for Default and PPDai, and $n_0 = 60{,}000$ for LendingClub.

\begin{figure}[h]
    \centering
    \includegraphics[scale=.4]{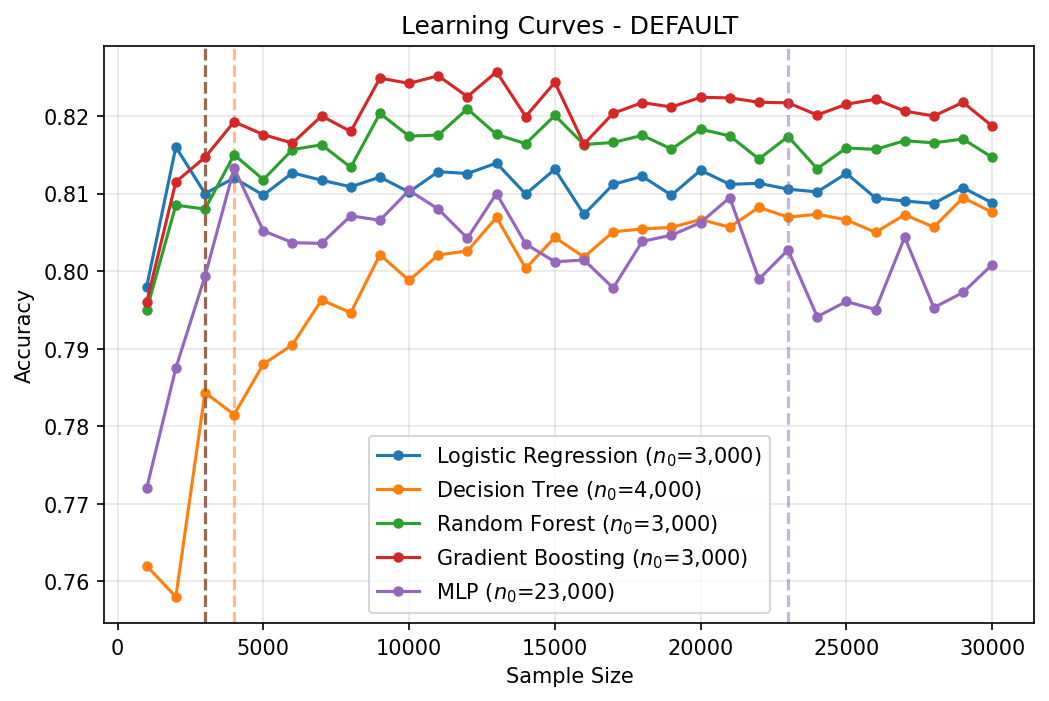}
    \caption{Learning curves for the Default dataset. Both RF and GBDT models stabilize around $n_0 = 5{,}000$.}
    \label{fig:app-n0-default}
\end{figure}

\section{\revision{Theoretical Analysis of the Accuracy--Recall Paradox: Deferred Proofs}}
\phantomsection %
\label{app:theory}

This appendix contains the proofs of the results stated in Section~\ref{subsec:theory}.
Throughout, assumptions A1--A6 and the notation ($\pi$, $\pi_t$, $\mathrm{prec}_t$, $\beta_t$, $\alpha_t$, $\mathcal{D}$, $\mathcal{D}_t$) are as defined there. The proof of Proposition~\ref{prop:decomp} relies on the following metric definitions.

\paragraph{Metrics.}

With $\mathrm{TP}, \mathrm{FP}, \mathrm{FN}, \mathrm{TN}$ the confusion-matrix counts (where $y=1$ denotes the default class), the metrics used in the proofs are probabilities estimated by the corresponding empirical frequencies:
\[
\begin{aligned}
\mathrm{acc} &= P(\hat{y} = y) &&= \frac{\mathrm{TP}+\mathrm{TN}}{\mathrm{TP}+\mathrm{TN}+\mathrm{FP}+\mathrm{FN}}, \\[2pt]
\mathrm{rec} &= P(\hat{y} = 1 \mid y = 1) &&= \frac{\mathrm{TP}}{\mathrm{TP}+\mathrm{FN}}, \\[2pt]
\mathrm{spec} &= P(\hat{y} = 0 \mid y = 0) &&= \frac{\mathrm{TN}}{\mathrm{TN}+\mathrm{FP}}, \\[2pt]
\mathrm{prec} &= P(y = 1 \mid \hat{y} = 1) &&= \frac{\mathrm{TP}}{\mathrm{TP}+\mathrm{FP}}.
\end{aligned}
\]

\begin{proof}[Lemma~\ref{lem:cleaning}]
By the law of total probability, conditioning on the prediction $\hat{y}$, where $\hat{y}=M_{t-1}(x)$:
\[
\begin{aligned}
\pi \;=\; P_{\mathcal{D}}(y=1)
&\;=\; \underbrace{P_{\mathcal{D}}(\hat{y}=1)}_{\alpha_t}\,\underbrace{P_{\mathcal{D}}(y=1 \mid \hat{y}=1)}_{\mathrm{prec}_t}
\;+\; \underbrace{P_{\mathcal{D}}(\hat{y}=0)}_{1-\alpha_t}\,\underbrace{P_{\mathcal{D}}(y=1 \mid \hat{y}=0)}_{\beta_t} \\[2pt]
&\;=\; \alpha_t\,\mathrm{prec}_t + (1-\alpha_t)\,\beta_t .
\end{aligned}
\]
Solving for $\beta_t$ and subtracting $\pi$,
\[
\beta_t - \pi \;=\; \frac{\pi - \alpha_t\,\mathrm{prec}_t}{1-\alpha_t} - \pi \;=\; \frac{\alpha_t\,(\pi - \mathrm{prec}_t)}{1-\alpha_t}.
\]
By A4, $\alpha_t \in (0,1)$, so $\alpha_t > 0$ and $1-\alpha_t > 0$; by A5, $\pi - \mathrm{prec}_t < 0$. Hence $\beta_t - \pi < 0$, i.e.\ $\beta_t < \pi$.
\end{proof}

\begin{proof}[Proposition~\ref{prop:onecycle}]
By the Biased recursion (Section~\ref{subsec:baselines}), $\mathcal{D}_1 = \mathcal{D}_0 \cup \{(\mathbf{x}_i,y_i) : \mathbf{x}_i\in \mathcal{S}_1 , \hat{y}_i = 0\}$, with
$\hat{y}_i = M_0(x_i)$: the initial training set $\mathcal{D}_0$ together with the applicants accepted from the first batch of arrivals $\mathcal{S}_1$. Let $N_0 = |\mathcal{D}_0|$ and $m_1 = |\{\mathbf{x}_i \in \mathcal{S}_1 : \hat{y}_i = 0\}| > 0$ (A4) denote their sizes; $\beta_1$ is the default rate of the accepted subset as defined in Lemma~\ref{lem:cleaning}. The pool default rate is the convex combination
\[
\pi_1 \;=\; \frac{N_0\,\pi_0 + m_1\,\beta_1}{N_0 + m_1},
\]
with strictly positive weights. By A6, $\pi_0 = \pi$, and by Lemma~\ref{lem:cleaning}, $\beta_1 < \pi$. A convex combination of $\pi$ and a value strictly below $\pi$, with positive weight on the latter, is strictly below $\pi$. Hence $\pi_1 < \pi = \pi_0$.
\end{proof}

\begin{proof}[Proposition~\ref{prop:decomp}]
Partition the accumulated pool by the true label and apply the law of total probability, with $\hat{y} = M_t(x)$, the predictions of the model trained on
$\mathcal{D}_t$ and $P_{\mathcal{D}_t}(y=1) = \pi_t$:
\[
\begin{aligned}
\mathrm{acc}_t \;=\; P_{\mathcal{D}_t}(\hat{y}=y)
&\;=\; P_{\mathcal{D}_t}(\hat{y}=0 \mid y=0)\,(1-\pi_t) \;+\; P_{\mathcal{D}_t}(\hat{y}=1 \mid y=1)\,\pi_t \\[2pt]
&\;=\; (1-\pi_t)\,\mathrm{spec}_t \;+\; \pi_t\,\mathrm{rec}_t.
\end{aligned}
\]
The contribution of recall is the term $\pi_t\,\mathrm{rec}_t \le \pi_t \to 0$, which establishes the stated sensitivity.
\end{proof}

We now connect the formal results above to the multi-cycle empirical dynamics reported in Section~\ref{subsec:baselines} using the recursion governing $\pi_t$ across cycles.
Lemma~\ref{lem:cleaning} gives $\beta_t < \pi$ at every cycle. Each cycle adds a disjoint accepted batch (size $m_t$, default rate $\beta_t$) to the pool, so defaulter counts add and the pool rate is the size-weighted average
\[
\pi_t = \frac{|\mathcal{D}_{t-1}|\,\pi_{t-1} + m_t\,\beta_t}{|\mathcal{D}_{t-1}| + m_t},
\qquad\text{thus}\qquad
\pi_t - \pi_{t-1} = \frac{m_t\,(\beta_t - \pi_{t-1})}{|\mathcal{D}_{t-1}| + m_t}.
\]
The right-hand form shows the pool deflates further ($\pi_t - \pi_{t-1} < 0$) exactly when the accepted batch has a lower default rate than the pool it joins ($\beta_t < \pi_{t-1}$), and is unchanged ($\pi_t = \pi_{t-1}$) when $\beta_t = \pi_{t-1}$. At $t=1$ the requirement is $\beta_1 < \pi$, supplied by Lemma~\ref{lem:cleaning} (Proposition~\ref{prop:onecycle}); from $t=2$ onward, once $\pi_{t-1} < \pi$, it becomes the stricter $\beta_t < \pi_{t-1}$ --- the accepted batch must fall below the already-deflated pool, not merely below the population. The experiments resolve whether this holds. Empirically (Figure~\ref{fig:default-rate-kickout}a) $\pi_t$ decreases monotonically from the population rate to roughly $17\%$ by $t=10$ on Default/GBDT.
In this case, as the pool is depleted of defaulters, the retrained model increasingly predicts non-default,  which empirically drives $\mathrm{rec}_t$ downward. By Proposition~\ref{prop:decomp}, recall contributes to accuracy only through $\pi_t\,\mathrm{rec}_t$, whose weight $\pi_t$ shrinks with the pool, so accuracy is largely decoupled from recall and can stay high as recall collapses (Figure~\ref{fig:biased-oracle}).
The multi-cycle decrease of $\pi_t$ and the observed trends in accuracy and recall are reported as empirical findings in Section~\ref{subsec:baselines}.

\section{Reject Inference: Additional Experimental Results}
\label{app:ri-additional-results}

This appendix reports the full set of results for all dataset--model configurations not presented in the main text. The qualitative patterns described in Section~\ref{sec:reject-inference} are consistent across all configurations: Extrapolation's metric dominance, the middle-tier behavior of Fuzzy Augmentation and Twins, the underperformance of Parceling and Shallow Self-Learning, and the Oracle Paradox, in which the Biased baseline attains higher accuracy than Oracle. We organize results by dataset, presenting each model family (GBDT, Random Forest) side by side.

\begin{figure}[h]
    \centering
    \includegraphics[width=\textwidth]{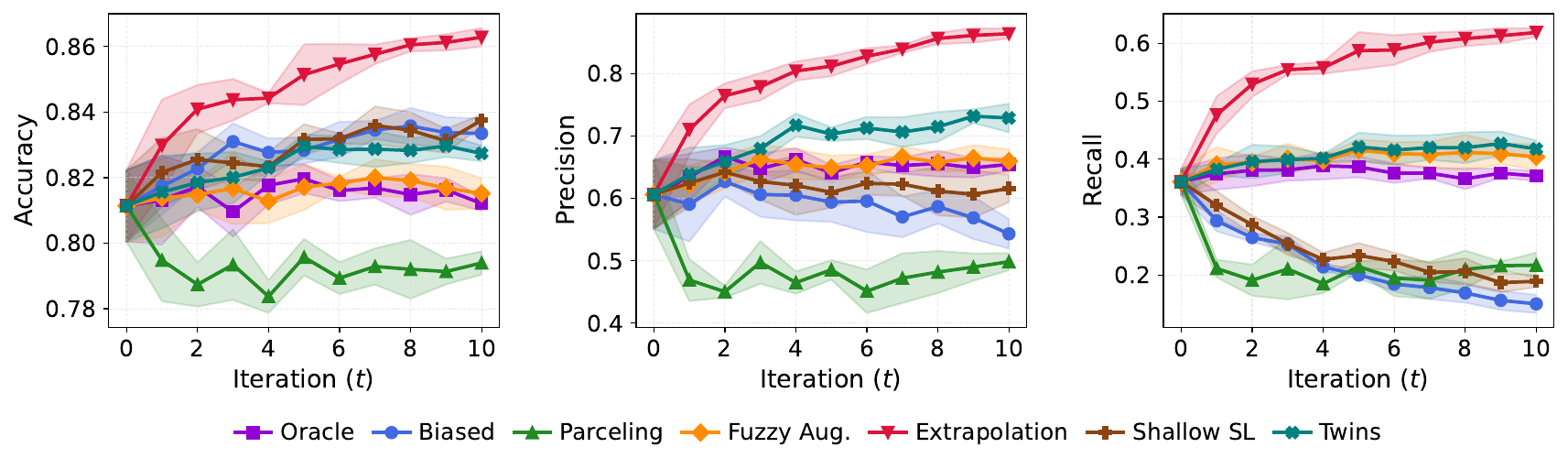}
    \caption{Accuracy, precision, and recall for all strategies on Default/RF ($n_0{=}5{,}000$, $\tau{=}0.5$). The same tier structure emerges as with GBDT (Figure~\ref{fig:ri-metrics}): Extrapolation dominates standard metrics, the middle tier (Fuzzy, Twins) clusters together, and Parceling and Shallow SL underperform.}
    \label{fig:app-default-rf-metrics}
\end{figure}

\subsection{Default --- Random Forest}
\label{app:default-rf}

Figure~\ref{fig:app-default-rf-metrics} shows the accuracy, precision, and recall trajectories for Default/RF. The same tier structure emerges as with GBDT: Extrapolation dominates standard metrics while Fuzzy Augmentation and Twins form a middle tier. The training set default rate and Kickout (Figure~\ref{fig:app-default-rf-dr-kickout}) show the same pattern as on GBDT: Extrapolation's metric advantage is accompanied by a training set default rate well above the population rate.

\begin{figure}
    \centering
    \includegraphics[width=\linewidth]{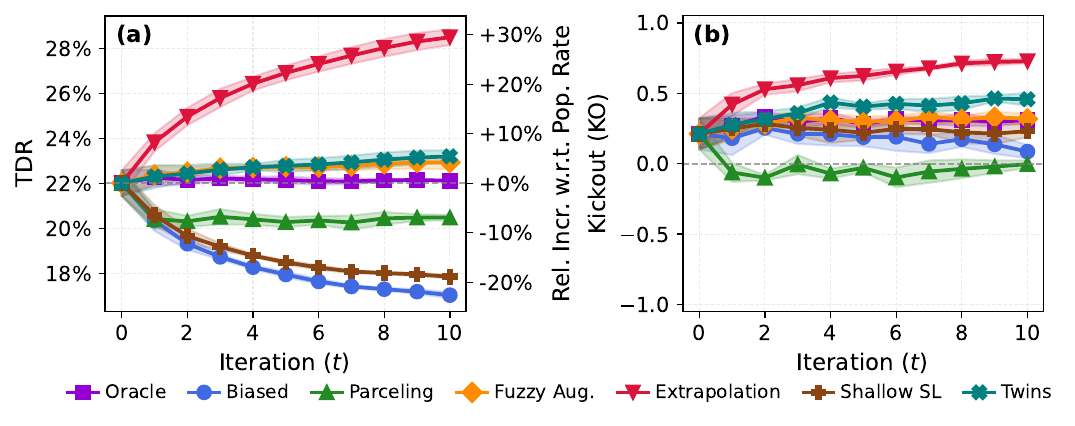}
    \caption{Training set default rate and Kickout for Default/RF. Extrapolation inflates the default rate above the population rate (dashed line) while Biased deflates it, mirroring the GBDT results (Figure~\ref{fig:default-rate-kickout}).}
    \label{fig:app-default-rf-dr-kickout}
\end{figure}

\subsection{LendingClub --- GBDT}
\label{app:lendingclub-gbdt-extra}

Figure~\ref{fig:app-lendingclub-gbdt-dr-kickout} reports the training set default rate and Kickout for LendingClub/GBDT. The same patterns hold as on Default, with compressed magnitudes due to the larger initial training set; the corresponding accuracy, precision, and recall comparison appears in Appendix~\ref{app:ri-dataset-size} (Figure~\ref{fig:lendingclub-ri}).

\begin{figure}[t]
    \centering
    \includegraphics[width=\linewidth]{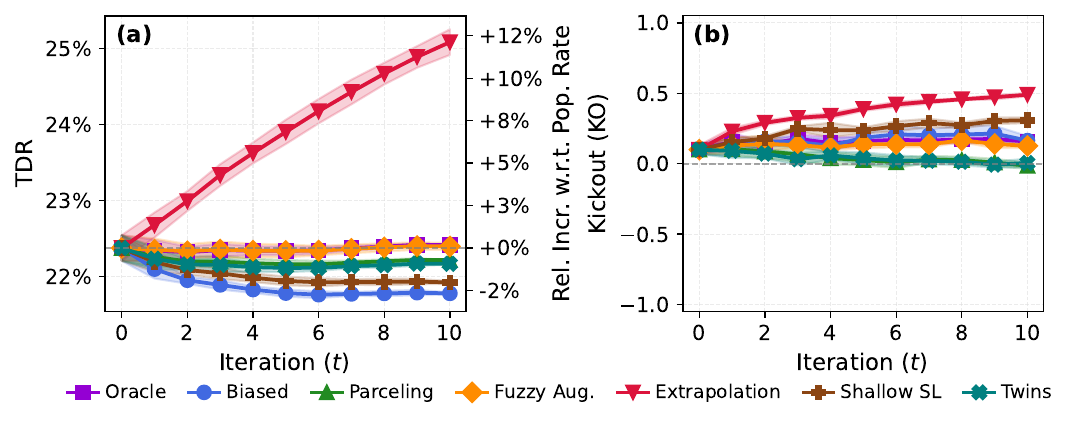}
    \caption{Training set default rate and Kickout for LendingClub/GBDT ($n_0{=}60{,}000$). The same patterns hold as on Default, but with compressed magnitudes due to the larger initial training set.}
    \label{fig:app-lendingclub-gbdt-dr-kickout}
\end{figure}

\subsection{LendingClub --- Random Forest}
\label{app:lendingclub-rf}

Figure~\ref{fig:app-lendingclub-rf-metrics} shows the full strategy comparison for LendingClub/RF. The patterns are consistent with the GBDT results, with similarly compressed magnitudes. The default rate and Kickout trajectories (Figure~\ref{fig:app-lendingclub-rf-dr-kickout}) confirm that Extrapolation inflates the training default rate while Biased deflates it.

\begin{figure}[t]
    \centering
    \includegraphics[width=\textwidth]{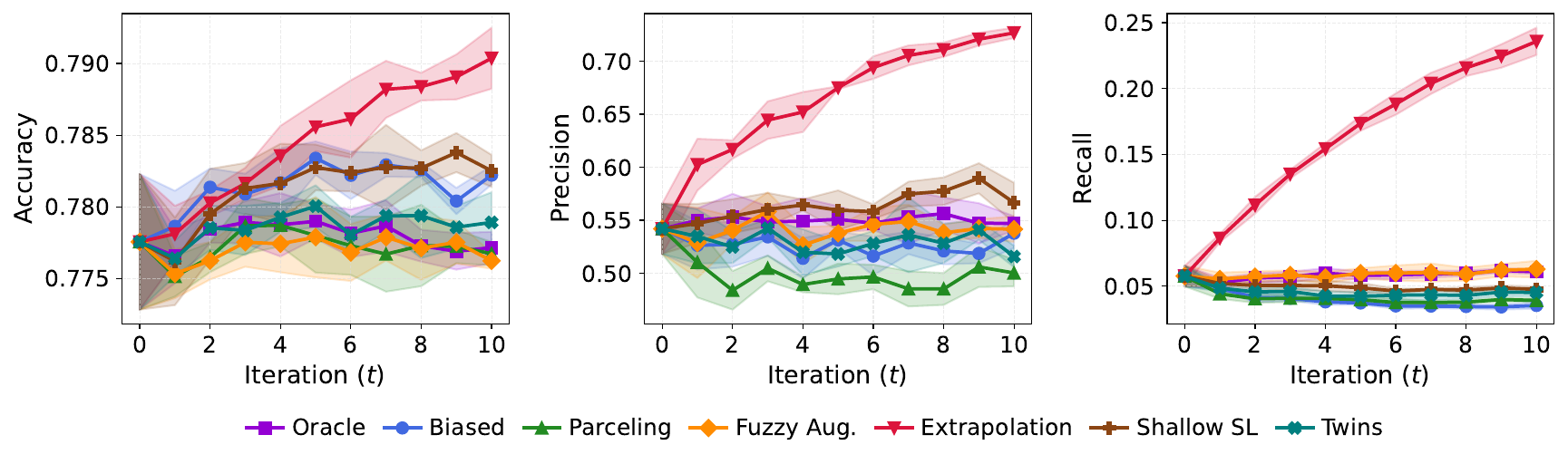}
    \caption{Accuracy, precision, and recall for all strategies on LendingClub/RF ($n_0{=}60{,}000$, $\tau{=}0.5$).}
    \label{fig:app-lendingclub-rf-metrics}
\end{figure}

\begin{figure}
    \centering
    \includegraphics[width=\linewidth]{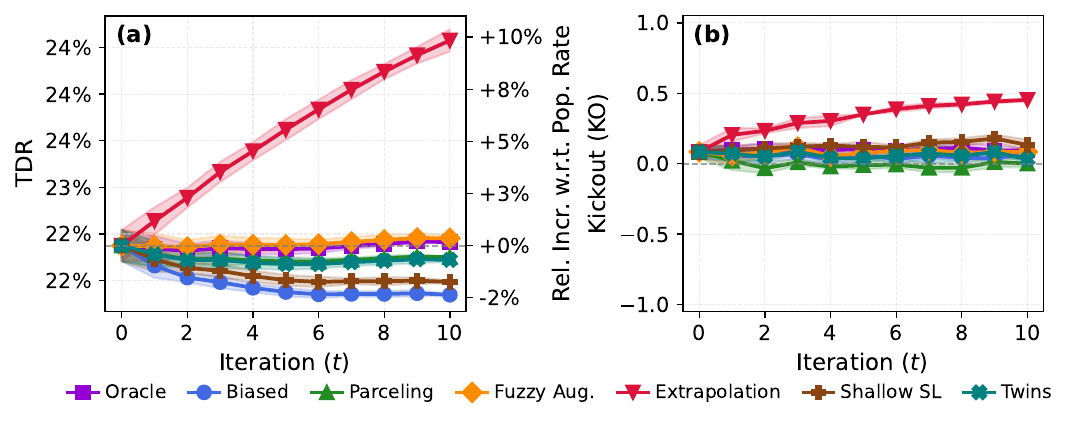}
    \caption{Training set default rate and Kickout for LendingClub/RF.}
    \label{fig:app-lendingclub-rf-dr-kickout}
\end{figure}

\subsection{PPDai --- GBDT}
\label{app:ppdai-gbdt}

Figure~\ref{fig:app-ppdai-gbdt-metrics} shows the strategy comparison for PPDai/GBDT. The qualitative patterns match those observed on Default, with Extrapolation dominating standard metrics and the middle tier (Fuzzy, Twins) providing the best balance of recall improvement and proximity to the population default rate. The default rate and Kickout plots (Figure~\ref{fig:app-ppdai-gbdt-dr-kickout}) show Extrapolation inflating the default rate above the population rate of 23.0\%.

\begin{figure}[t]
    \centering
    \includegraphics[width=\textwidth]{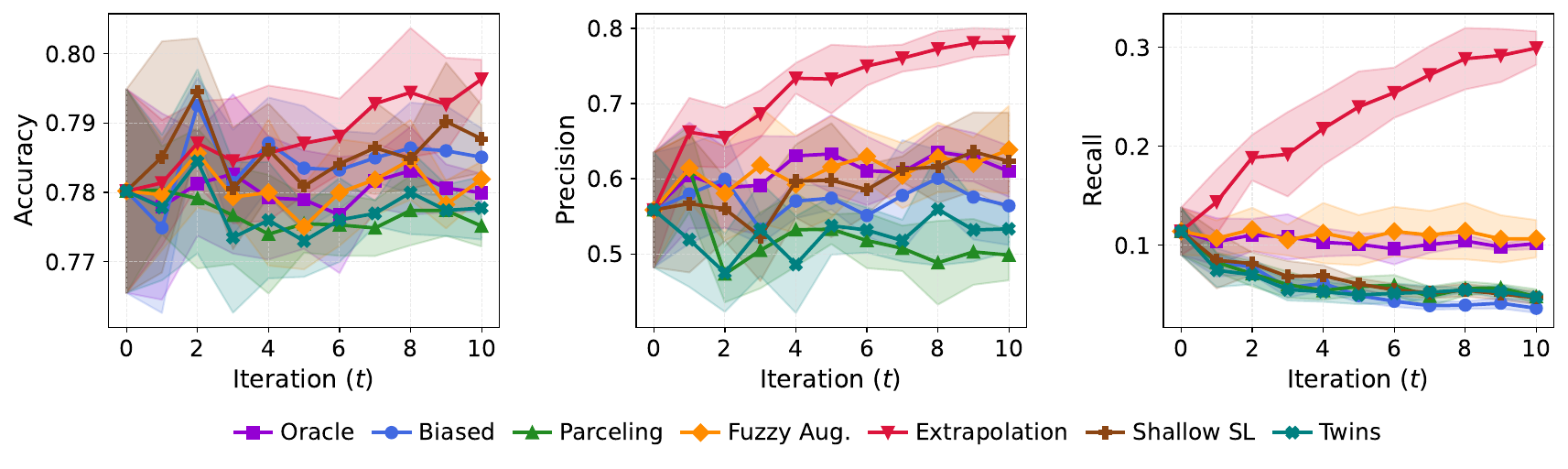}
    \caption{Accuracy, precision, and recall for all strategies on PPDai/GBDT ($n_0{=}5{,}000$, $\tau{=}0.5$).}
    \label{fig:app-ppdai-gbdt-metrics}
\end{figure}

\begin{figure}
    \centering
    \includegraphics[width=\linewidth]{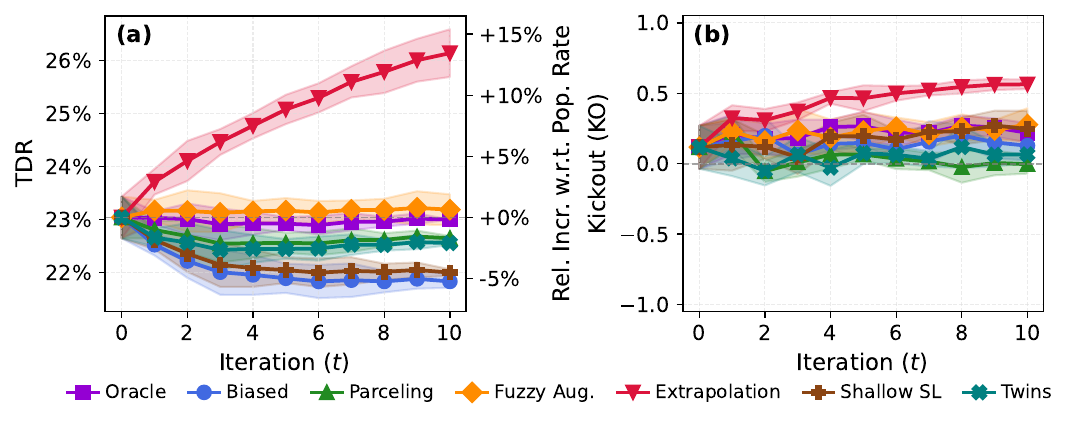}
    \caption{Training set default rate and Kickout for PPDai/GBDT. The population default rate is 23.0\% (dashed line).}
    \label{fig:app-ppdai-gbdt-dr-kickout}
\end{figure}

\begin{figure}[t]
    \centering
    \includegraphics[width=\textwidth]{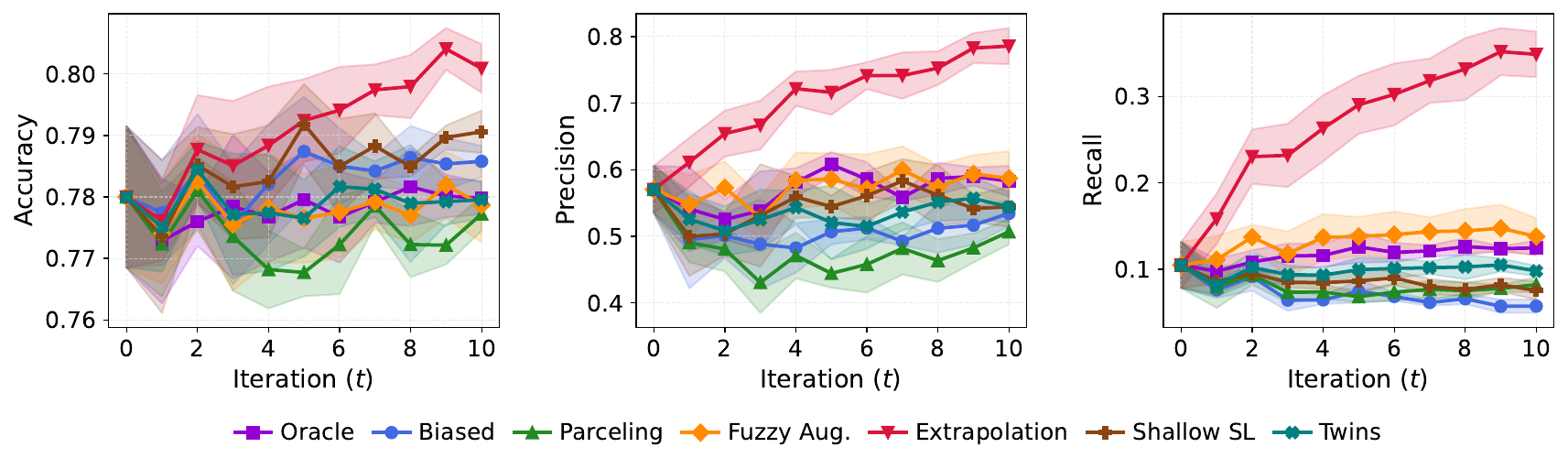}
    \caption{Accuracy, precision, and recall for all strategies on PPDai/RF ($n_0{=}5{,}000$, $\tau{=}0.5$).}
    \label{fig:app-ppdai-rf-metrics}
\end{figure}

\subsection{PPDai --- Random Forest}
\label{app:ppdai-rf}

Figure~\ref{fig:app-ppdai-rf-metrics} shows the strategy comparison for PPDai/RF. The results are consistent with the GBDT configuration.  Notably, the Kickout plot (Figure~\ref{fig:app-ppdai-rf-dr-kickout}) shows Parceling reaching negative values, meaning it rejects more non-defaulters than defaulters. This is the weakest rejection quality of any strategy in our experiments, though it remains well above the reference point for label-independent rejection on this dataset (given by $2\pi - 1 = -0.54$ for PPDai with $\pi=23\%$).

\begin{figure}
    \centering
    \includegraphics[width=\linewidth]{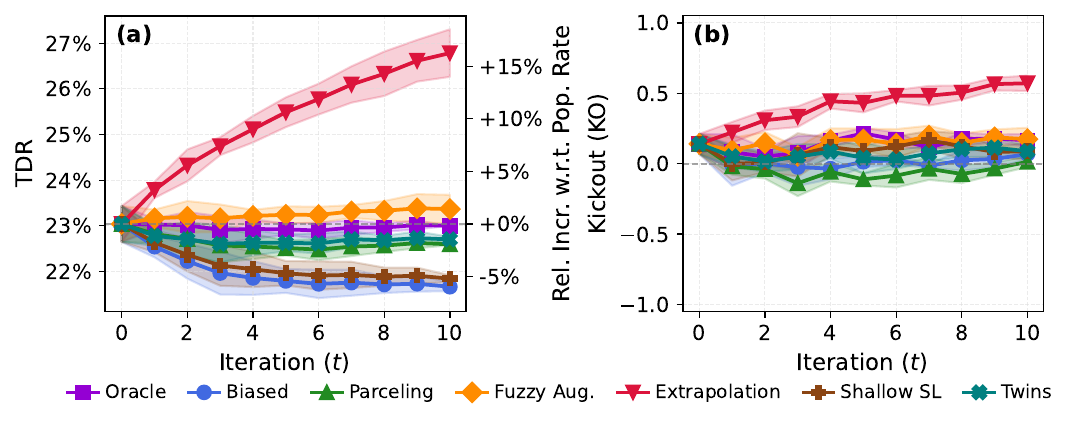}
    \caption{Training set default rate and Kickout for PPDai/RF. Parceling's Kickout becomes negative on this configuration, meaning it rejects more non-defaulters than defaulters, and is the weakest value across our experiments.}
    \label{fig:app-ppdai-rf-dr-kickout}
\end{figure}

\subsection{Effect of Dataset Size}
\label{app:ri-dataset-size}

The magnitude of survival bias effects varies systematically with dataset size (Figure~\ref{fig:lendingclub-ri}). On LendingClub ($n_0{=}60{,}000$), all effects are more compressed: Biased recall drops to approximately $0.02$--$0.04$ but the absolute accuracy differences between strategies are small ($0.776$--$0.796$). On Default and PPDai ($n_0{=}5{,}000$), the same patterns emerge with larger magnitudes: Biased recall falls to $0.04$--$0.15$, and Extrapolation's accuracy exceeds the next-best strategy's by 2.3 and 2.5 percentage points on Default, against 0.9 and 1.0 on PPDai and 1.1 and 0.8 on LendingClub. This scale effect has practical significance. Larger lenders with extensive historical portfolios experience slower bias accumulation and may not detect the problem within a small number of retraining cycles. Smaller lenders, niche market segments, and new product lines, where initial training data is more limited, are most vulnerable to the survival bias spiral. The temporal dynamics documented in Section~\ref{sec:reject-inference} would manifest faster and with greater severity in precisely those contexts.

\begin{figure}[t]
    \centering
    \includegraphics[width=\textwidth]{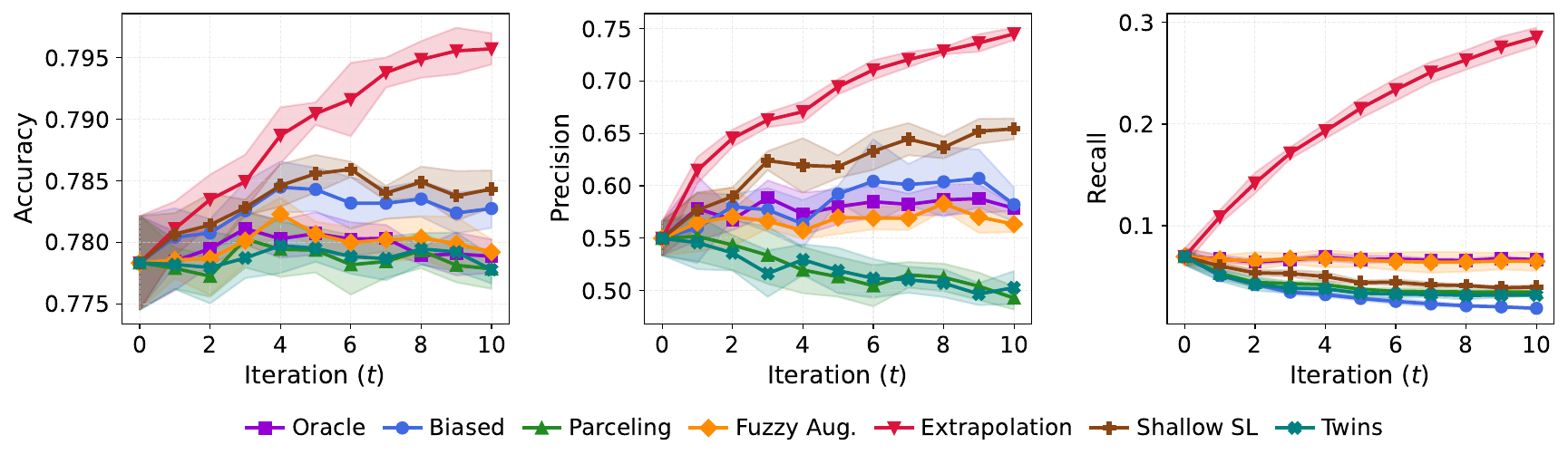}
    \caption{Strategy comparison on LendingClub (GBDT, $n_0{=}60{,}000$, $\tau{=}0.5$). The same qualitative patterns emerge as on Default (Figure~\ref{fig:ri-metrics}), but with compressed magnitudes: accuracy differences between strategies are smaller ($0.778$--$0.796$ range) and recall degradation is less severe. Larger initial training sets buffer against survival bias accumulation.}
    \label{fig:lendingclub-ri}
\end{figure}

\section{Additional Controlled Exploration Results}
\label{app:exploration}

The main text (Section~\ref{sec:exploration}) reports controlled exploration results for Default/GBDT. Figures~\ref{fig:app-ranking-default-rf}--\ref{fig:app-ranking-lendingclub-rf} show ranking strategy comparisons for the remaining dataset--model configurations. 
The ordering holds on average rather than at every rate: Most Risky attains the highest mean Kickout across exploration rates in all six configurations, and Least Risky the lowest, but neither dominates at every individual rate. Random provides a robust middle ground, requiring no model scores.
Magnitudes vary with dataset scale, as discussed in Section~\ref{sec:exploration-tradeoff}.

The calibration pattern reported in Section~\ref{sec:ranking-comparison} varies across configurations. At $r = 0.10$, averaged over iterations, the mean predicted default probability among randomly selected rejects is within five points of the observed rate in all six, but the direction differs: 
it exceeds the observed rate on Default/RF and falls below it on LendingClub/GBDT. 
Over-prediction among the highest-scored rejects is clear on both Default configurations and on PPDai/GBDT, weak on PPDai/RF and LendingClub/RF, and absent on LendingClub/GBDT, the configuration in which every effect we report is most attenuated.

\begin{figure}[htb]
  \centering
  \begin{subfigure}[b]{0.48\linewidth}
    \centering
    \includegraphics[width=\linewidth]{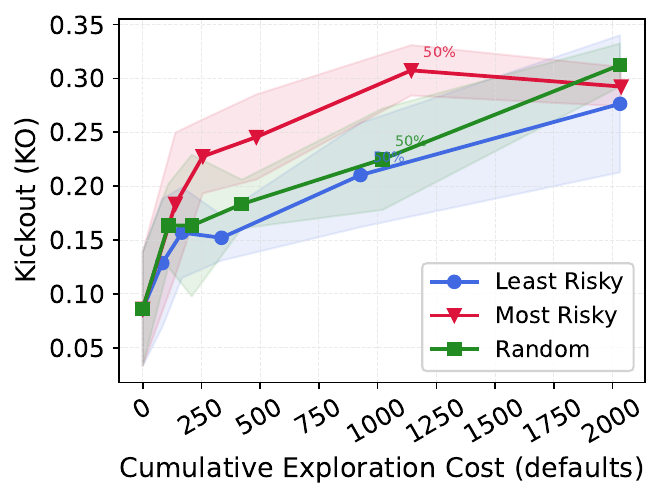}
    \caption{Kickout vs.\ cumulative cost.}
  \end{subfigure}
  \hfill
  \begin{subfigure}[b]{0.48\linewidth}
    \centering
    \includegraphics[width=\linewidth]{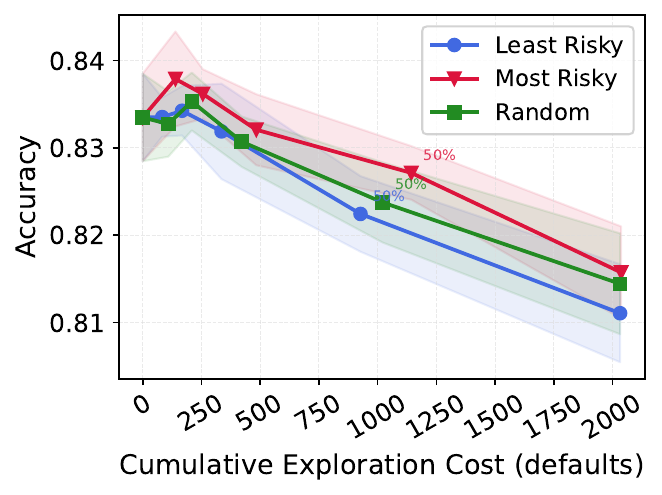}
    \caption{Accuracy vs.\ cumulative cost.}
  \end{subfigure}
  \caption{Ranking strategy comparison on the Default dataset (RF). The pattern mirrors Default/GBDT (Figure~\ref{fig:ranking-comparison}), with a slightly larger accuracy decline ($-1.9\%$ vs.\ $-1.5\%$ at $r = 1.0$).}
  \label{fig:app-ranking-default-rf}
\end{figure}

\begin{figure}[htb]
  \centering
  \begin{subfigure}[b]{0.48\linewidth}
    \centering
    \includegraphics[width=\linewidth]{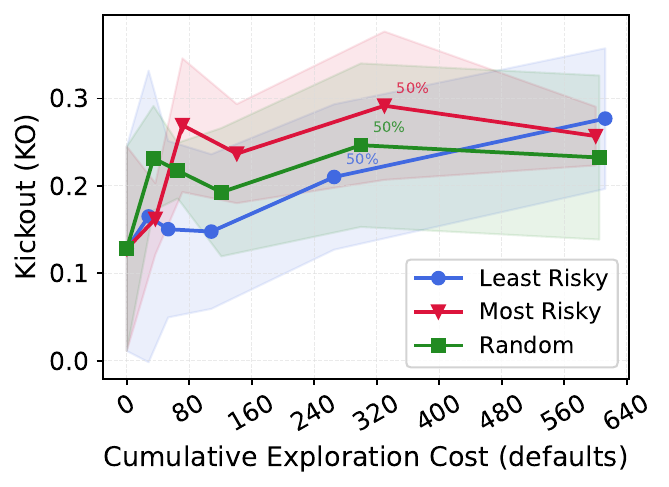}
    \caption{Kickout vs.\ cumulative cost.}
  \end{subfigure}
  \hfill
  \begin{subfigure}[b]{0.48\linewidth}
    \centering
    \includegraphics[width=\linewidth]{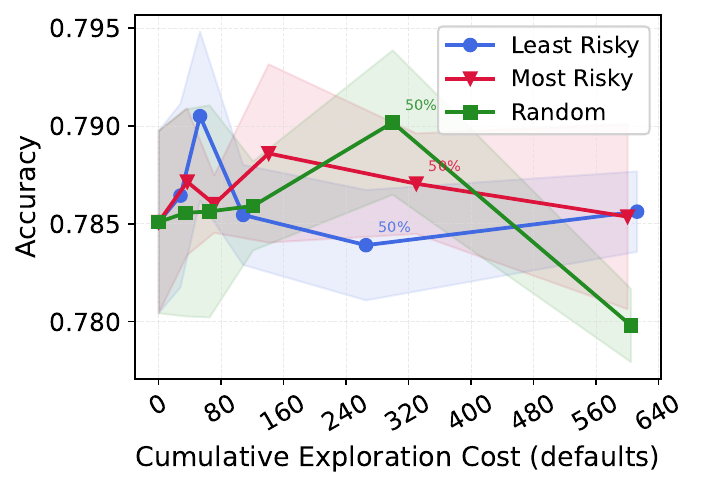}
    \caption{Accuracy vs.\ cumulative cost.}
  \end{subfigure}
  \caption{Ranking strategy comparison on the PPDai dataset (GBDT). Effects are compressed relative to Default, reflecting the smaller TDR gap ($1.17\%$).}
  \label{fig:app-ranking-ppdai-gbdt}
\end{figure}

\begin{figure}[htb]
  \centering
  \begin{subfigure}[b]{0.48\linewidth}
    \centering
    \includegraphics[width=\linewidth]{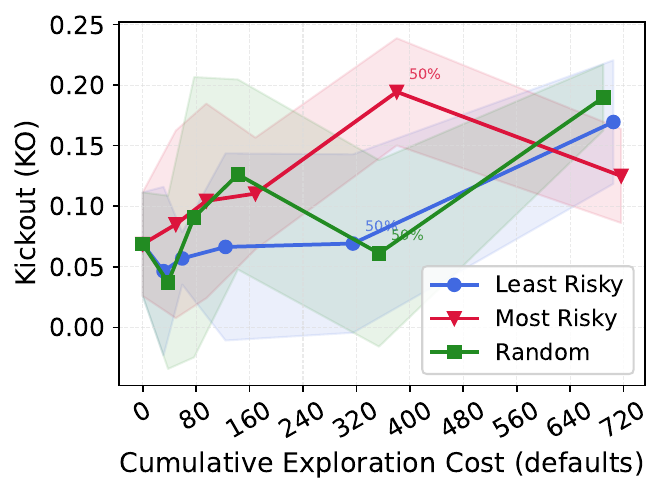}
    \caption{Kickout vs.\ cumulative cost.}
  \end{subfigure}
  \hfill
  \begin{subfigure}[b]{0.48\linewidth}
    \centering
    \includegraphics[width=\linewidth]{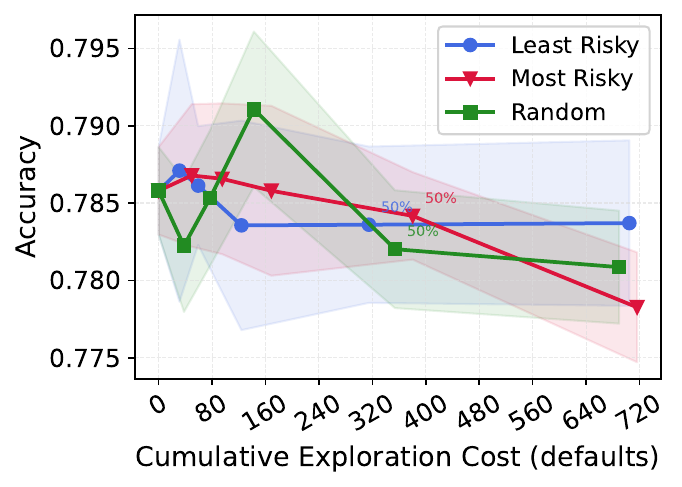}
    \caption{Accuracy vs.\ cumulative cost.}
  \end{subfigure}
  \caption{Ranking strategy comparison on the PPDai dataset (RF). Kickout is non-monotone in the exploration rate, with Most Risky peaking at $r=0.50$ before falling back.}
  \label{fig:app-ranking-ppdai-rf}
\end{figure}

\begin{figure}[htb]
  \centering
  \begin{subfigure}[b]{0.48\linewidth}
    \centering
    \includegraphics[width=\linewidth]{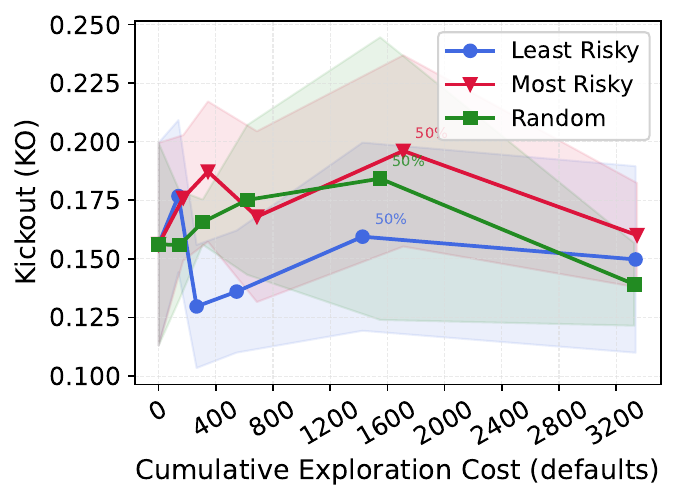}
    \caption{Kickout vs.\ cumulative cost.}
  \end{subfigure}
  \hfill
  \begin{subfigure}[b]{0.48\linewidth}
    \centering
    \includegraphics[width=\linewidth]{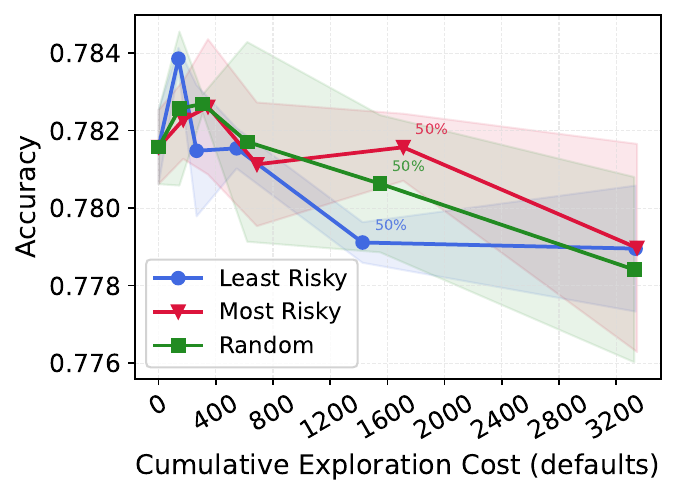}
    \caption{Accuracy vs.\ cumulative cost.}
  \end{subfigure}
  \caption{Ranking strategy comparison on the LendingClub dataset (GBDT). Kickout shows no systematic improvement at any exploration rate, consistent with the small TDR gap ($0.64\%$).}
  \label{fig:app-ranking-lendingclub-gbdt}
\end{figure}

\begin{figure}[htb]
  \centering
  \begin{subfigure}[b]{0.48\linewidth}
    \centering
    \includegraphics[width=\linewidth]{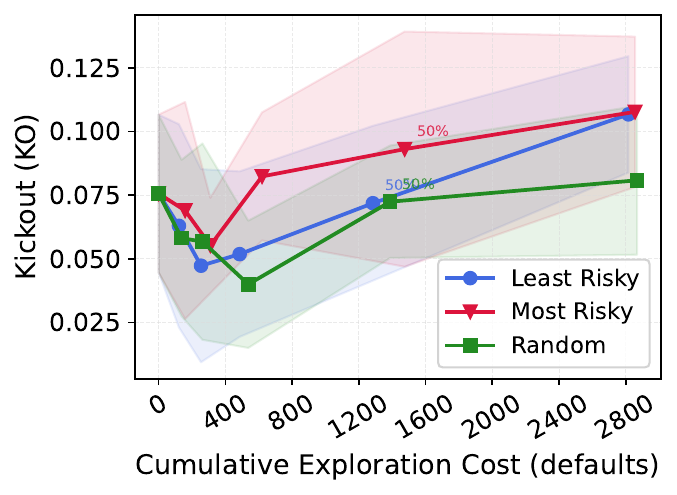}
    \caption{Kickout vs.\ cumulative cost.}
  \end{subfigure}
  \hfill
  \begin{subfigure}[b]{0.48\linewidth}
    \centering
    \includegraphics[width=\linewidth]{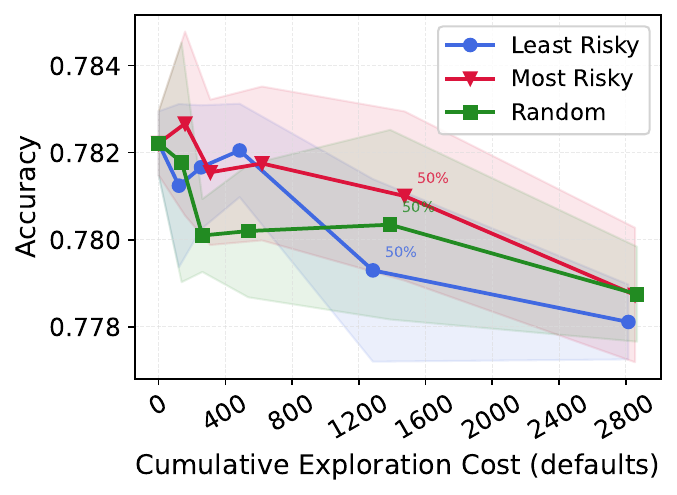}
    \caption{Accuracy vs.\ cumulative cost.}
  \end{subfigure}
  \caption{Ranking strategy comparison on the LendingClub dataset (RF). Effects are the most attenuated of the three datasets: accuracy declines by $0.35\%$, while Kickout falls slightly below its baseline at intermediate rates and recovers to at most a $0.032$ gain at full exploration. This is consistent with the minimal TDR distortion ($2.6\%$).}
  \label{fig:app-ranking-lendingclub-rf}
\end{figure}

\end{document}